\documentclass[letterpaper]{article} 
\usepackage{AAAI_Template/aaai2026}  
\newif\ifarxiv
\arxivtrue
\ifarxiv\nocopyright\fi
\usepackage{times}  
\usepackage{helvet}  
\usepackage{courier}  
\usepackage[hyphens]{url}  
\usepackage{graphicx} 
\usepackage{natbib}  
\usepackage{caption} 
\usepackage{algorithm}
\usepackage{algorithmic}
\usepackage{xspace}
\usepackage{amssymb}
\usepackage{amsmath}
\usepackage{amsthm}
\usepackage{mathtools}
\usepackage{etoc}
\usepackage{cleveref}

\usepackage{tikz}
\usepackage{amsmath, amssymb}
\usetikzlibrary{positioning, arrows.meta, fit, backgrounds, calc}
\usepackage{pgfplots}
\pgfplotsset{compat=1.18}
\usepackage{csvsimple}


\newcommand{\appsubsubsection}[1]{%
  \subsubsection{#1}%
  \addcontentsline{toc}{subsubsection}{\protect\numberline{\thesubparagraph}#1}%
}
\usepackage{newfloat}
\usepackage{listings}
\DeclareCaptionStyle{ruled}{labelfont=normalfont,labelsep=colon,strut=off} 
\floatstyle{ruled}
\newfloat{listing}{tb}{lst}{}
\floatname{listing}{Listing}
\newtheorem{theorem}{Theorem}
\newtheorem{lemma}{Lemma}
\newtheorem{proposition}{Proposition}
\newtheorem{corollary}{Corollary}

\theoremstyle{definition}
\newtheorem{definition}{Definition}

\theoremstyle{remark}
\newtheorem{remark}{Remark}

\newcommand{\dynsys}{\ensuremath{\mathcal{D}}\xspace}
\newcommand{\init}{\ensuremath{\mathbf{I}\xspace}}
\newcommand{\trans}{\ensuremath{\mathbf{F}\xspace}}
\newcommand{\noise}{\ensuremath{\mathbf{E}\xspace}}
\newcommand{\ctrl}{\ensuremath{\mathbf{u}\xspace}}
\newcommand{\timestep}{\ensuremath{\delta}\xspace}
\newcommand{\horizon}{\ensuremath{T}\xspace}
\newcommand{\reach}{\ensuremath{\mathbf{G}\xspace}}
\newcommand{\avoid}{\ensuremath{\mathbf{A}\xspace}}

\newcommand{\state}{\ensuremath{\mathbf{x}\xspace}}
\newcommand{\State}{\ensuremath{\mathbf{X}\xspace}}
\renewcommand{\next}{\ensuremath{\mathit{next}}\xspace}
\newcommand{\err}{\ensuremath{\mathbf{\epsilon}\xspace}}
\newcommand{\conv}{\ensuremath{\textbf{conv}\xspace}}
\newcommand{\aff}{\ensuremath{\textbf{aff}\xspace}}
\newcommand{\svert}{\ensuremath{\textbf{vert}\xspace}}
\newcommand{\trajj}{\ensuremath{\mathit{\tau}}\xspace}
\newcommand{\vleq}{\ensuremath{\preceq}\xspace}
\newcommand{\vgeq}{\ensuremath{\succeq}\xspace}
\newcommand{\ReLU}{\ensuremath{\mathsf{ReLU}}\xspace}

\newcommand{\nats}{\ensuremath{\mathbb{N}}\xspace}
\newcommand{\reals}{\ensuremath{\mathbb{R}}\xspace}

\newcommand{\bset}{\ensuremath{B}\xspace}
\newcommand{\pset}{\ensuremath{P}\xspace}
\newcommand{\gridPt}{\ensuremath{\Vec{p}}\xspace}
\newcommand{\enc}{\ensuremath{\mathcal{E}}\xspace}
\newcommand{\Ltri}{\ensuremath{L_{\Delta}}\xspace}
\newcommand{\Utri}{\ensuremath{U_{\Delta}}\xspace}
\newcommand{\Lc}{\ensuremath{L_{\square}}\xspace}
\newcommand{\Uc}{\ensuremath{U_{\square}}\xspace}
\newcommand{\scomp}{\ensuremath{\Delta}\xspace}
\newcommand{\poly}{\ensuremath{\mathcal{P}}\xspace}
\newcommand{\dom}{\ensuremath{\mathit{dom}}\xspace}
\newcommand{\cells}{\ensuremath{\mathit{Cells}}\xspace}

\newcommand{\netDom}{\ensuremath{\mathcal{X}}\xspace}
\newcommand{\al}{\ensuremath{\mathbf{a}_L}\xspace}
\newcommand{\au}{\ensuremath{\mathbf{a}_U}\xspace}
\newcommand{\bl}{\ensuremath{b_L}\xspace}
\newcommand{\bu}{\ensuremath{b_U}\xspace}
\newcommand{\x}{\ensuremath{\mathbf{x}}\xspace}
\newcommand{\ep}{\ensuremath{\mathbf{\epsilon}}\xspace}
\newcommand{\p}{\ensuremath{\mathbf{p}}\xspace}
\newcommand{\simp}{\ensuremath{\mathcal{S}}\xspace}
\newcommand{\cell}{\ensuremath{\mathcal{C}}\xspace}
\newcommand{\brel}{\ensuremath{B_\square}\xspace}

\newcommand{\clipper}{\textsc{clipper}\xspace}
\newcommand{\bart}{\textsc{bart}\xspace}
\newcommand{\rer}{\textsc{rer}\xspace}
\newcommand{\rail}{\textsc{rail}\xspace}
\newcommand{\muni}{\textsc{muni}\xspace}
\newcommand{\nfg}{\ensuremath{\mathcal{G}^{\dynsys}_\horizon}\xspace}
\title{Closing the Loop: Branch-and-Bound for Scalable Verification of Nonlinear Neural Feedback Systems}
\author{
    I. Samuel Akinwande\textsuperscript{\rm 1},
    Mykel J. Kochenderfer\textsuperscript{\rm 1},
    Clark Barrett\textsuperscript{\rm 2}
}
\affiliations{
    \textsuperscript{\rm 1}Department of Aeronautics and Astronautics, Stanford University\\
    \textsuperscript{\rm 2}Department of Computer Science, Stanford University\\
    Stanford, CA, 94305, USA\\
    samakin@stanford.edu, mykel@stanford.edu, barrett@cs.stanford.edu
}

\usepackage{bibentry}

\begin{document}

\maketitle

\begin{abstract}
Despite recent advances in the verification of nonlinear neural feedback systems, scalability remains the central obstacle, as state-of-the-art solvers do not yet handle the network sizes and nonlinear dynamics of autonomy applications.
Combinatorial solvers do not scale to large networks, whereas propagative solvers excessively sacrifice precision.
This work seeks to improve the scalability of combinatorial solvers by formulating verification as branch-and-bound on an abstraction of the closed-loop system. We introduce \rail, an interface that exposes polyhedral enclosures of the dynamics to LiRPA-style bound propagation, and \clipper, a branch-and-bound algorithm that jointly refines enclosures and splits controller activations.
This framework enables joint reasoning on the computational graph of the closed-loop system, preserving symbolic correlations across time steps.
We present our construction and show that it yields significant improvements over the state of the art.
\end{abstract}

\section{Introduction}
Autonomy is increasingly integral to modern engineered systems, many of which are dynamical systems operating in a closed loop with neural network controllers. Such \emph{neural feedback systems} (NFS) are prevalent in safety-critical domains including aerial navigation~\cite{kaufmann2023champion}, autonomous vehicles~\cite{ettinger2021large}, and legged robotics~\cite{rodriguez2022neural}, motivating principled approaches to safety verification.
The complex interactions between the system dynamics and the neural controller make safety verification challenging, and this challenge is heightened when the dynamics are nonlinear.
In this work, we develop techniques that apply to both linear and nonlinear dynamics, but our exposition focuses on nonlinear dynamics as the harder and more relevant setting.

Current approaches to NFS verification fall broadly into \emph{propagative} and \emph{combinatorial} families~\cite{akinwande2026polyhedralenclosuresefficientcombinatorial}. Propagative solvers represent nonlinear dynamics using control-native abstractions such as Taylor models~\cite{dutta2019reachability,huang2022polar}, Bernstein polynomials~\cite{fan2020reachnn}, polynomial zonotopes~\cite{kochdumper2023open}, and inclusion functions~\cite{harapanahalli2024immrax}. Monolithic variants such as CORA~\cite{kochdumper2023open}, ReachNN~\cite{fan2020reachnn}, and POLAR~\cite{huang2022polar} encode the neural network directly in these abstractions, enabling joint reasoning but limiting neural-network analysis to control-native machinery. Compositional variants such as immrax~\cite{harapanahalli2024immrax} and CROWN-Reach~\cite{manzanas_lopez2024arch,crown_reach} analyze the network and the dynamics separately, gaining scalability on the network but requiring reasoning across distinct abstractions.
Combinatorial solvers (OVERTVerify~\cite{sidrane2022overt}, OvertPoly~\cite{akinwande2026polyhedralenclosuresefficientcombinatorial}) encode both the dynamics and the network as mixed-integer linear programs (MILPs), enabling joint reasoning over a shared abstraction. Like the monolithic propagative tools, they trade scalability for structural uniformity, though here the shared abstraction is network-native rather than control-native. MILP encodings do not scale to large networks, and requiring exact MILP-encodable activations restricts the class of admissible controllers.

Open-loop neural-network verification (certifying the network in isolation rather than in closed loop with the system dynamics) faces neither barrier, as state-of-the-art verifiers combine linear bound propagation (LiRPA)~\cite{xu2020automatic}, often derived from dual LP formulations~\cite{de2021improved}, with branch-and-bound (BaB) over activation-pattern prefixes~\cite{bunel2020branch,xu2020fast}. Bounds are computed in parallel on GPUs, scaling to large networks. This combination closely mirrors MILP solver strategies and suggests compatibility with the abstractions used by combinatorial NFS solvers. 

The key research question in this work is whether open- and closed-loop verification can be unified in a single, scalable framework. We answer affirmatively, with three contributions: $(i)$ \rail, a \textbf{unified graph interface}, exposes polyhedral enclosures of the dynamics to LiRPA, enabling joint reasoning over system dynamics and network activations; $(ii)$ \clipper, a \textbf{structure-aware branch-and-bound} algorithm, jointly refines these enclosures and splits controller activations, with branching heuristics that exploit the structure of both the dynamics and the network while leveraging GPU parallelism; and $(iii)$ an implementation for verifying reach-avoid specifications for nonlinear NFS at \textbf{competitive precision using GPU-accelerated BaB} yields significant improvements over the state of the art.

\section{Background}
\subsubsection{Notation and Definitions}
For a set $\State$, $2^{\State}$ denotes its power set. For integers $i \leq j$, $[i..j] \coloneqq \{ z \mid i \leq z \leq j \}$, and $[n] \coloneqq [1..n]$. If $S = (s_1, \ldots, s_n)$ is a finite sequence, then $S_i$ denotes its $i$-th element. We use $\conv$ to denote the convex hull operator, \nats the nonnegative integers, and \reals the real numbers.
For $\mathbf{v}, \mathbf{w} \in \reals^n$, we write $\mathbf{v} \vleq \mathbf{w}$ to denote $\mathbf{v}_i \leq \mathbf{w}_i$ for every $i \in [n]$.
For a function $f : X \to Y$ and a subset $S \subseteq X$, we write $f|_S$ for the restriction of $f$ to $S$. If a set is used where a point is expected, the function is applied to each element and the resulting sets are unioned.
We use several geometric terms informally and provide formal definitions in Appendix~\ref{app:geometric-defs}.

\subsection{Neural Feedback Systems}
A neural feedback system is a discrete-time dynamical system governed by a neural network controller. Its trajectories are assessed against a goal set $\reach$ and unsafe states $\avoid$.
Formally, the tuple $\langle m,n, \init, \trans, \noise, \ctrl, B, \timestep, \horizon, \reach, \avoid \rangle$ defines a neural feedback system $\dynsys$, where $m,n \in \nats$ are the controller-output ($\ctrl(\state) \in \reals^m$) and state dimensions ($\state \in \reals^n$), $\init \subseteq \reals^n$ the initial states, $\trans = (f_1, \ldots, f_n)$ a vector field with $f_i : \reals^n \to \reals$, $\noise \subseteq \reals^n$ admissible perturbations, $\ctrl : \reals^n \to \reals^m$ a neural network controller, $B \in \reals^{n \times m}$ a controller-input matrix, $\timestep \in \reals$ the time step, $\horizon \in \nats$ the time horizon, $\reach \subseteq \reals^n$ the goal set, and $\avoid : [0..\horizon] \to 2^{\reals^n}$ unsafe states at each step.

The system evolves over $\horizon$ discrete time steps, with total time horizon $\horizon \cdot \timestep$. The next-state function $\next^{\dynsys} : \reals^n \to 2^{\reals^n}$ is defined by
\begin{align}
  \next^{\dynsys}(\state)
  &\coloneqq \{ \state + (\trans(\state) + B\ctrl(\state) + \err)\,\timestep \mid \err \in \noise \}.
  \label{eq:next_state}
\end{align}

For $\State_0 \subseteq \init$, the trajectory $\trajj^{\dynsys}(\State_0) \coloneqq (\State_0, \ldots, \State_\horizon)$ is a sequence of state sets, where $\State_t \coloneqq \next^{\dynsys}(\State_{t-1})$ for $t \in [\horizon]$. The system $\dynsys$ is \emph{safe} if
\begin{align}
  &\forall \state \in \init.\ \exists t \in [0..\horizon].\ \trajj^{\dynsys}(\{\state\})_t \subseteq \reach \label{eq:reach_prop} \\
  &\forall t \in [0..\horizon].\ \trajj^{\dynsys}(\init)_t \cap \avoid(t) = \emptyset \label{eq:avoid_prop}
\end{align}
Equation~\eqref{eq:reach_prop} is the reach property, requiring all trajectories to reach $\reach$, while \eqref{eq:avoid_prop} is the avoid property, requiring all trajectories to avoid unsafe states at each $t$.

\subsection{Polyhedral Enclosures}
Polyhedral enclosures are a combinatorial abstraction for nonlinear NFS verification~\cite{akinwande2026polyhedralenclosuresefficientcombinatorial,akinwande2026fabricstrategyverifyingneural}, underpinning state-of-the-art combinatorial solvers. A bounding set is a tuple $\bset = \langle n, \pset, L, U \rangle$, where $\pset \subseteq \reals^n$ is a grid and $L, U : \pset \to \reals$ are lower and upper bounding functions. Component names follow the bounding set's decoration, so the grid of $\bset^j$ is $\pset_j$.
Given a Delaunay triangulation $\Delta$ of $\pset$ (see Appendix~\ref{app:triangulating_grids}), we obtain a collection of $(n+1)$-dimensional polytopes by lifting each $n$-simplex's $n+1$ vertices $\gridPt$ to the pairs $(\gridPt, L(\gridPt))$ and $(\gridPt, U(\gridPt))$ and taking their convex hull. The \emph{polyhedral enclosure} $\enc(\bset,\Delta) \subseteq \reals^n \times \reals$ is the union of these per-simplex polytopes.

Let $D \coloneqq \conv(\pset)$. The lower and upper surfaces of $\enc(\bset, \Delta)$ define functions $\Ltri, \Utri : D \to \reals$, affine on each simplex of $\Delta$. We say $\bset$ \emph{encloses} $f : D \to \reals$ if $\Ltri(\state) \leq f(\state) \leq \Utri(\state)$ for all $\state \in D$ and every Delaunay triangulation $\Delta$ of $\pset$. Formal definitions appear in Appendix~\ref{app:poly_encs}.
\subsection{Linear Relaxation-Based Perturbation Analysis}
Linear relaxation-based perturbation analysis (LiRPA)~\cite{xu2020automatic}, the basis of state-of-the-art neural network verification, computes linear bounds on network outputs. Given a neural network $\pi : \netDom \to \reals^o$ and $\State \subseteq \netDom$, LiRPA produces $\underline{\mathbf{W}}_o, \overline{\mathbf{W}}_o$ and $\underline{\mathbf{b}}_o, \overline{\mathbf{b}}_o$ such that
\begin{align}
  \forall \state \in \State.\quad \underline{\mathbf{W}}_o \state + \underline{\mathbf{b}}_o \;\vleq\; \pi(\state) \;\vleq\; \overline{\mathbf{W}}_o \state + \overline{\mathbf{b}}_o.
\end{align}
Bounds are computed by composing per-node linear relaxations, which are exact at affine layers and sound but inexact at nonlinearities (see \Cref{prop:lirpa}). Slope parameters $\alpha$ at piecewise-linear activations can be optimized via gradient methods, yielding bounds competitive with LP relaxations ($\alpha$-CROWN) while retaining efficient GPU implementation. When $\State$ is an $\ell_p$-norm ball, the linear bounds concretize to scalar bounds via Hölder's inequality without explicit optimization over $\State$. Formal definitions appear in Appendix~\ref{app:lirpa}.

\subsection{Branch and Bound for Neural Network Verification}
LiRPA bounds are sound but may be too loose to certify a property on their own, leaving verification incomplete. State-of-the-art verifiers recover completeness by combining LiRPA with branch-and-bound (BaB)~\citep{xu2020fast}. BaB first partitions the verification problem into subproblems by splitting the input set $\State$ or fixing the activation pattern of an unstable neuron (one whose pre-activation can be positive or negative over $\State$), then computes LiRPA bounds on each, and finally, either certifies the property once every subproblem is verified, or refutes it once a subproblem yields a counterexample. Bound propagation parallelizes across subproblems, enabling efficient GPU-based implementations. Na\"ive neuron splits may not tighten LiRPA bounds directly, since LiRPA bounds each neuron independently and cannot propagate split constraints. $\beta$-CROWN~\citep{wang2021beta} addresses this by incorporating split constraints into bound propagation via Lagrangian relaxation, with multipliers $\beta \geq 0$ jointly optimized alongside $\alpha$. The split dimension or neuron is selected by a \emph{branching heuristic}. A detailed treatment appears in Appendix~\ref{app:bab}.

\section{The CLIPPER Algorithm}
\label{sec:clipper}
This section develops \clipper in four stages. We first introduce an algorithm called \bart for constructing \emph{relaxed bounding sets}, affine enclosures that bound a polyhedral enclosure over its hyperrectangular domain while leaving slope parameters free. We then define \emph{recursive enclosure refinement} (\rer), which tightens these bounds by subdividing cells. \clipper is then a branch-and-bound procedure over graphs, whose nodes are either neurons in a neural network or polyhedral enclosures constructed using the \rer and \bart algorithms. We call the nodes in this graph \rail nodes.
\Cref{fig:nfg} shows one time step of such a graph. In what follows, proofs are omitted due to space constraints (they can be found in Appendix~\ref{app:proofs}).

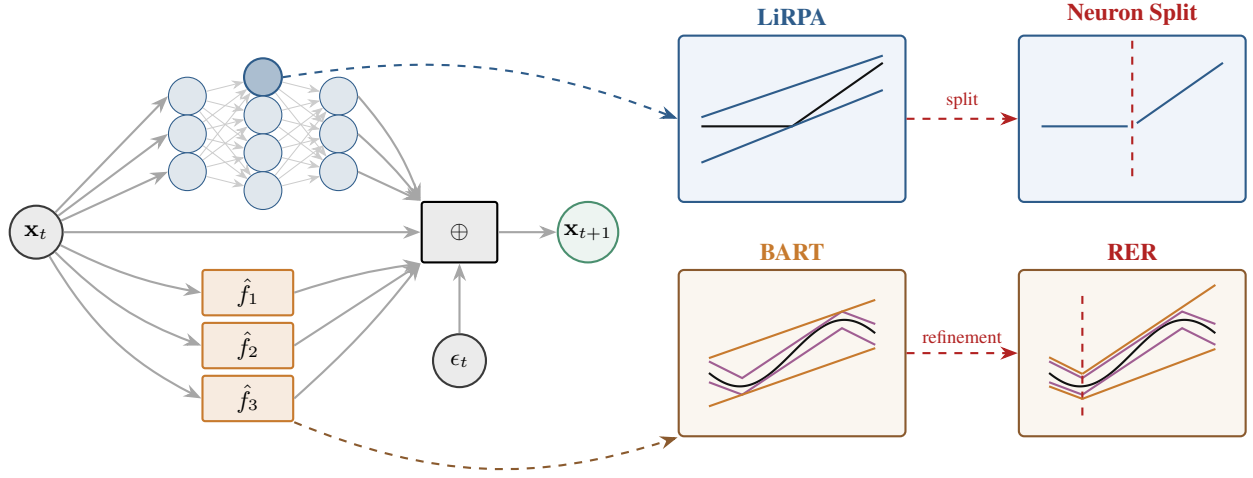
\begin{figure*}[t]
    \centering
{%
\definecolor{nfgCtrl}{HTML}{2E5C8A}
\definecolor{nfgRail}{HTML}{C77A2E}
\definecolor{nfgPwl}{HTML}{A05E92}      
\definecolor{nfgNonlin}{HTML}{111111}
\definecolor{nfgAction}{HTML}{B22222}
\definecolor{nfgLirpaBg}{HTML}{F0F4FA}
\definecolor{nfgLirpaBorder}{HTML}{2E5C8A}
\definecolor{nfgCalloutBg}{HTML}{FAF6F0}
\definecolor{nfgCalloutBorder}{HTML}{8A5A2E}
\definecolor{nfgInput}{HTML}{3A3A3A}
\definecolor{nfgState}{HTML}{4A8F6F}

\begin{tikzpicture}[
    >=Stealth,
    font=\small,
    every node/.style={font=\small},
    input/.style={circle, draw=nfgInput, fill=nfgInput!10,
                  minimum size=7mm, inner sep=0pt, thick},
    state/.style={circle, draw=nfgState, fill=nfgState!10,
                  minimum size=8mm, inner sep=0pt, thick},
    neuron/.style={circle, draw=nfgCtrl, fill=nfgCtrl!15,
                   minimum size=5mm, inner sep=0pt},
    actneuron/.style={circle, draw=nfgCtrl, fill=nfgCtrl!40,
                      minimum size=5mm, inner sep=0pt, thick},
    railnode/.style={rectangle, draw=nfgRail, fill=nfgRail!15,
                     minimum width=12mm, minimum height=6mm,
                     thick, rounded corners=1pt},
    combiner/.style={rectangle, draw=black, fill=black!8,
                     minimum width=10mm, minimum height=8mm,
                     thick, rounded corners=1pt},
    edge/.style={->, thick, gray!70},
    callout/.style={rectangle, draw=nfgCalloutBorder, fill=nfgCalloutBg,
                    thick, rounded corners=2pt, inner sep=4pt},
    lirpacallout/.style={rectangle, draw=nfgLirpaBorder, fill=nfgLirpaBg,
                         thick, rounded corners=2pt, inner sep=4pt},
    splitcallout/.style={rectangle, draw=nfgLirpaBorder, fill=nfgLirpaBg,
                         thick, rounded corners=2pt, inner sep=4pt},
]

\node[input] (xt) at (0, 0) {$\x_t$};

\begin{scope}[shift={(2.0, 1.3)}]
    \node[neuron]    (c-i1) at (0, 0.5)  {};
    \node[neuron]    (c-i2) at (0, 0)    {};
    \node[neuron]    (c-i3) at (0, -0.5) {};
    \node[actneuron] (c-h1) at (1.0, 0.75)  {};
    \node[neuron]    (c-h2) at (1.0, 0.25)  {};
    \node[neuron]    (c-h3) at (1.0, -0.25) {};
    \node[neuron]    (c-h4) at (1.0, -0.75) {};
    \node[neuron]    (c-o1) at (2.0, 0.5)  {};
    \node[neuron]    (c-o2) at (2.0, 0)    {};
    \node[neuron]    (c-o3) at (2.0, -0.5) {};
    \foreach \i in {1,2,3} {
        \foreach \h in {1,2,3,4} {
            \draw[edge, gray!40, line width=0.3pt, ->] (c-i\i) -- (c-h\h);
        }
    }
    \foreach \h in {1,2,3,4} {
        \foreach \o in {1,2,3} {
            \draw[edge, gray!40, line width=0.3pt, ->] (c-h\h) -- (c-o\o);
        }
    }
\end{scope}

\node[railnode] (rail1) at (2.8, -0.8) {$\hat{f}_1$};
\node[railnode] (rail2) at (2.8, -1.5) {$\hat{f}_2$};
\node[railnode] (rail3) at (2.8, -2.2) {$\hat{f}_3$};

\node[combiner] (combiner) at (5.6, 0)    {$\oplus$};
\node[input]    (et)       at (5.6, -1.7) {$\epsilon_t$};
\node[state]    (xt1)      at (7.3, 0)    {$\x_{t+1}$};

\draw[edge] (xt) -- (c-i1.west);
\draw[edge] (xt) -- (c-i2.west);
\draw[edge] (xt) -- (c-i3.west);
\draw[edge] (xt) to[bend right=8]  (rail1.west);
\draw[edge] (xt) to[bend right=12] (rail2.west);
\draw[edge] (xt) to[bend right=16] (rail3.west);
\draw[edge] (xt.east) -- (combiner.west);
\draw[edge] (c-o1.east) to[bend left=10] (combiner.north west);
\draw[edge] (c-o2.east) to[bend left=5]  (combiner.north west);
\draw[edge] (c-o3.east) to[bend left=0]  (combiner.north west);
\draw[edge] (rail1.east) to[bend left=5]  (combiner.south west);
\draw[edge] (rail2.east) to[bend left=0]  (combiner.south west);
\draw[edge] (rail3.east) to[bend right=5] (combiner.south west);
\draw[edge] (et) -- (combiner.south);
\draw[edge, thick] (combiner.east) -- (xt1.west);

\begin{scope}[shift={(10.0, -1.6)}]
    \node[callout, minimum width=3.0cm, minimum height=2.2cm] (bart-box) {};
    \node[above, font=\small\bfseries, nfgRail] at (bart-box.north) {BART};
    \begin{scope}[shift={(0, 0)}, xscale=1.1, yscale=1.1]
        \draw[thick, nfgNonlin] plot[smooth, domain=-1:1, samples=50]
            (\x, {0.4*sin(deg(2.5*\x))});
        \draw[thick, nfgPwl]
            (-1, -0.10) -- (-0.6, -0.30) -- (0, 0.10) -- (0.6, 0.50) -- (1, 0.35);
        \draw[thick, nfgPwl]
            (-1, -0.35) -- (-0.6, -0.50) -- (0, -0.10) -- (0.6, 0.30) -- (1, 0.10);
        \draw[thick, nfgRail] (-1, -0.06) -- (1,  0.64);
        \draw[thick, nfgRail] (-1, -0.64) -- (1,  0.06);
    \end{scope}
\end{scope}

\begin{scope}[shift={(14.5, -1.6)}]
    \node[callout, minimum width=3.0cm, minimum height=2.2cm] (rer-box) {};
    \node[above, font=\small\bfseries, nfgAction] at (rer-box.north) {RER};
    \begin{scope}[shift={(0, 0)}, xscale=1.1, yscale=1.1]
        \draw[thick, nfgNonlin] plot[smooth, domain=-1:1, samples=50]
            (\x, {0.4*sin(deg(2.5*\x))});
        \draw[thick, nfgPwl]
            (-1, -0.10) -- (-0.6, -0.30) -- (0, 0.10) -- (0.6, 0.50) -- (1, 0.35);
        \draw[thick, nfgPwl]
            (-1, -0.35) -- (-0.6, -0.50) -- (0, -0.10) -- (0.6, 0.30) -- (1, 0.10);
        \draw[thick, nfgAction, dashed] (-0.6, -0.75) -- (-0.6, 0.7);
        \draw[thick, nfgRail] (-1,   -0.05) -- (-0.6, -0.25);
        \draw[thick, nfgRail] (-1,   -0.40) -- (-0.6, -0.55);
        \draw[thick, nfgRail] (-0.6, -0.25) -- (1,     0.82);
        \draw[thick, nfgRail] (-0.6, -0.55) -- (1,     0.05);
    \end{scope}
\end{scope}

\begin{scope}[shift={(10.0, 1.5)}]
    \node[lirpacallout, minimum width=3.0cm, minimum height=2.2cm] (lirpa-box) {};
    \node[above, font=\small\bfseries, nfgLirpaBorder] at (lirpa-box.north) {LiRPA};
    \begin{scope}[shift={(0, -0.1)}, xscale=1.2, yscale=1.2]
        \draw[thick, nfgNonlin] (-1, 0) -- (0, 0) -- (1, 0.7);
        \draw[thick, nfgLirpaBorder] (-1,  0.10) -- (1, 0.78);
        \draw[thick, nfgLirpaBorder] (-1, -0.40) -- (1, 0.40);
    \end{scope}
\end{scope}

\begin{scope}[shift={(14.5, 1.5)}]
    \node[splitcallout, minimum width=3.0cm, minimum height=2.2cm] (nsplit-box) {};
    \node[above, font=\small\bfseries, nfgAction] at (nsplit-box.north) {Neuron Split};
    \begin{scope}[shift={(0, -0.1)}, xscale=1.2, yscale=1.2]
        \draw[thick, nfgAction, dashed] (0, -0.55) -- (0, 0.95);
        \draw[thick, nfgLirpaBorder] (-1, 0) -- (-0.05, 0);
        \draw[thick, nfgLirpaBorder] (0.05, 0.035) -- (1, 0.7);
    \end{scope}
\end{scope}

\draw[->, thick, nfgCalloutBorder, dashed]
    (rail3.south east) to[bend right=20] (bart-box.south west);

\draw[->, thick, nfgAction, dashed]
    (bart-box.east) --
    node[above, font=\scriptsize, nfgAction] {refinement}
    (rer-box.west);

\draw[->, thick, nfgLirpaBorder, dashed]
    (c-h1.east) to[bend left=15] (lirpa-box.west);

\draw[->, thick, nfgAction, dashed]
    (lirpa-box.east) --
    node[above, font=\scriptsize, nfgAction] {split}
    (nsplit-box.west);

\end{tikzpicture}
}%
\caption{Anatomy of one time step in the neural feedback graph $\nfg$. The state $\x_t$ feeds the controller $\pi_{(t)}$ (top) and dynamics $F^{(t)}$ (bottom), which are composed with an affine combiner with disturbance $\ep_t$ to produce $\x_{t+1}$. Each nonlinear node admits a sound linear relaxation (LiRPA at activations, \bart at \rail enclosure nodes) and a corresponding split or refinement action (neuron splits at activations, \rer refinement at \rail nodes). Callouts illustrate the relaxations (left column) and the actions (right column).}
\label{fig:nfg}
\end{figure*}

\subsection{Polyhedral Enclosure Relaxations}
Polyhedral enclosures and LiRPA share linear structure but operate over different domains. Polyhedral enclosures provide piecewise-affine bounds over simplices, while LiRPA takes linear bounds over $\ell_p$-norm balls (most commonly $\ell_\infty$) and propagates them through neural graphs. Since hyperrectangles generalize $\ell_\infty$ balls, unifying the two requires extending polyhedral enclosures from simplices to hyperrectangular domains.

\subsubsection{Relaxing Enclosures}
We extend polyhedral enclosures to hyperrectangular domains by introducing sound \emph{relaxations}. We call these relaxations \bart (\textbf{B}ounding-set \textbf{A}ffine \textbf{R}elaxation \textbf{T}echnique).
Consider a bounding set $\bset$. Recall that any Delaunay triangulation $\Delta$ of $\pset$ yields an enclosure $\enc(\bset, \Delta)$ whose lower and upper surfaces are piecewise-affine functions. Let $f: \reals^n \to \reals$ be a nonlinear function enclosed by $\bset$ so that $\Ltri(\x) \leq f(\x) \leq \Utri(\x)$.

\begin{definition}[Relaxed Bounding Set]
Let $\bset$ be a bounding set, and let $\al, \au \in \reals^n$. The \emph{relaxation} of $\bset$ with parameters $(\al, \au)$ is the bounding set $\bset_\square = \langle n, \pset_\square, L_\square, U_\square \rangle$, where $\pset_\square$ is the set of $2^n$ corners of $\conv(\pset)$, $L_\square(\p) = \al \cdot \p + \bl$ and $U_\square(\p) = \au \cdot \p + \bu$ with $\bl = \min_{\p \in \pset} (L(\p) - \al \cdot \p)$, and $\bu = \max_{\p \in \pset} (U(\p) - \au \cdot \p)$. We say a bounding set is \emph{relaxed} if it is the relaxation of some bounding set, and \emph{unrelaxed} otherwise.
\end{definition}
\noindent
The box labeled \bart in Figure 1 shows an example with $n=2$: the black curve represents $f$; the purple segments represent $B$; and the orange lines are $L_\square$ and $U_\square$.
We write $\Lc, \Uc$ also for the lower and upper surfaces of any enclosure based on $\bset_\square$: since \Lc and \Uc are affine, $\Lc(\x) = \al \cdot \x + \bl$ and $\Uc(\x) = \au \cdot \x + \bu$ on $\conv(\pset_\square) = \conv(\pset)$, they are the same for every enclosure $\enc(\bset_\square,\Delta)$, regardless of  $\Delta$. The following theorem establishes that $\bset_\square$ encloses $f$ over $\conv(\pset)$.

\begin{theorem}[Soundness of Relaxation]
\label{thm:relaxation}
If $f$ is enclosed by $\bset$, then for any $\al, \au \in \reals^n$, any $\x \in \conv(\pset)$, and any Delaunay triangulation $\Delta$ of $\pset$, $\Lc(\x) \leq \Ltri(\x) \leq f(\x) \leq \Utri(\x) \leq \Uc(\x)$.
\end{theorem}

This construction leaves $\al$ and $\au$ as free parameters that can be optimized with respect to an arbitrary objective (e.g., tightness). The optimization is amenable to first-order methods, analogous to $\alpha$ optimization in $\alpha$-CROWN. 
The relaxation is sound for any $\al, \au \in \reals^n$, with the tightness objective implicitly bounding the optimal parameters. As $\al$ varies, the affine bound $\Lc$ sweeps the supporting hyperplanes of the convex envelope of $\Ltri$ over $\conv(\pset)$. A single relaxation is thus exact only when $\Ltri$ is affine and loose otherwise. Refinement (next section) restores tightness by relaxing subregions independently, and is exact on any subregion where $\Ltri$ is affine.

Since relaxed bounding sets are sound (\Cref{thm:relaxation}), they compose under the same rules as unrelaxed enclosures, at substantially reduced encoding cost. 
\subsection{Refining Enclosure Relaxations}
The \bart relaxation loosens the tight bounds generated by polyhedral enclosures. When necessary, a degree of tightness can be recovered by partitioning the domain into subregions on which the relaxation is applied independently. This forms the foundation for the branch-and-bound algorithm we develop later.
\subsubsection{Refining Relaxations}
Relaxation produces a bounding set $\bset_\square$ whose enclosure is sound but may be excessively loose, particularly when the curvature of the enclosed function changes over the domain, or when the domain is large. We tighten such enclosures via the \emph{refinement} operation, which subdivides a bounding set into pieces over sub-hyperrectangles of the domain. 
We call this refinement operation \rer (\textbf{R}ecursive \textbf{E}nclosure \textbf{R}efinement).
Each piece can be relaxed independently, yielding a piecewise-affine bound that is no looser than the single relaxation and strictly tighter whenever either piece attains a tighter offset~(\Cref{thm:refinement}).
\begin{definition}[Enclosure Refinement]
Let $\bset$ be a bounding set, and let $H$ be a \emph{grid hyperplane} of $\pset$ (an axis-aligned hyperplane that splits \pset at an interior grid coordinate, formalized in Appendix~\ref{app:grid_hyperplane}). Let $\pset_1, \pset_2$ be the sub-grids of $\pset$ on the two closed sides of $H$, so that $\pset_1 \cup \pset_2 = \pset$ and $\pset_1 \cap \pset_2 = \pset \cap H$. Then $\bset^1 := \langle n, \pset_1, L|_{\pset_1}, U|_{\pset_1} \rangle$ and $\bset^2 := \langle n, \pset_2, L|_{\pset_2}, U|_{\pset_2} \rangle$ are \emph{refinements} of $\bset$ along $H$.
\end{definition}
Refinement tightens enclosures only when combined with relaxation. Refining an unrelaxed bounding set preserves the original enclosure exactly, whereas relaxing the refined pieces can produce a tighter enclosure than the single relaxation on regions where it is not exact.

\begin{theorem}[Soundness of Refinement]
\label{thm:refinement}
Let $\bset$ be a bounding set, let $H$ be a grid hyperplane of $\pset$, and let $\bset^1, \bset^2$ be the refinements of $\bset$ along $H$. Let $\Delta$ be a Delaunay triangulation of $\pset$, and let
\[
    \Delta_j \coloneqq \{\simp \in \Delta \mid \simp \subseteq \conv(\pset_j)\} \quad \text{for } j \in \{1, 2\}.
\]
Each $\Delta_j$ is a Delaunay triangulation of $\pset_j$ by the Delaunay splitting lemma (see Appendix~\ref{app:delaunay_split}). Then:
\begin{enumerate}
    \item \textbf{(Unrelaxed.)} $\enc(\bset^1, \Delta_1) \cup \enc(\bset^2, \Delta_2) = \enc(\bset, \Delta)$.

    \item \textbf{(Relaxed.)} Let $\al, \au \in \reals^n$, let $\bset_\square$ and $\bset_\square^j$ be the relaxations of $\bset$ and $\bset^j$ with the same parameters $(\al, \au)$, and let $\bl^j, \bu^j$ denote the offsets of $\bset_\square^j$. For any Delaunay triangulations $\Delta_\square$ of $\pset_\square$ and $\Delta_\square^j$ of $\pset_\square^j$,
    \[
        \enc(\bset_\square^1, \Delta_\square^1) \cup \enc(\bset_\square^2, \Delta_\square^2) \subseteq \enc(\bset_\square, \Delta_\square),
    \]
    with strict containment whenever $(\bl^j, \bu^j) \neq (\bl, \bu)$ for some $j \in \{1, 2\}$.
\end{enumerate}
\end{theorem}
\subsection{\rail: Composing Enclosures with LiRPA}
The bounds (\Lc and \Uc) of a relaxed bounding set are affine, matching LiRPA's
native representation. A relaxed bounding set $\brel$ can be
embedded in a LiRPA computational graph as a node with known linear bounds, and
LiRPA's propagation rules can be used to propagate these bounds through the network. We call the
resulting approach \rail (\textbf{R}elaxed \textbf{A}ffine \textbf{I}nterface
for \textbf{L}iRPA). \bart and \rer bound the nonlinear dynamics, and LiRPA
propagates the bounds.
\begin{remark} \bart and \rer relax a multivariate nonlinear function $f : \reals^n \to \reals$ as a single affine enclosure over a hyperrectangular domain, refinable along any grid hyperplane. LiRPA provides no such relaxation. It relaxes a fixed library of primitive operators (e.g., activation functions) and composes the results through the network. \bart and \rer thus extend sound linear relaxation to coupled multivariate functions. \end{remark}
\subsection{Neural Feedback Graphs}
A neural feedback system unrolls over its horizon into a computational graph. This view is implicit in the dependency-graph formalism of
OvertPoly~\cite{akinwande2026polyhedralenclosuresefficientcombinatorial}, which encodes each node as
MILP constraints. The contribution here is to realize the dynamics as \rail nodes, exposing the graph to bound propagation and joint
branch-and-bound over enclosure refinements and neuron
splits.

\begin{definition}[Neural Feedback Graph]\label{def:nfg} Let $\dynsys$ be a neural feedback system. Its \emph{neural feedback graph} \nfg is a directed acyclic graph with input nodes $\x_0 \in \init$ and $\ep_0, \ldots, \ep_{\horizon-1} \in \noise$, and $\horizon$ per-step subgraphs chained through state nodes $\x_1, \ldots, \x_\horizon$. The subgraph at time $t$ takes $\x_t$ and $\ep_t$ as inputs and realizes $\next^\dynsys$ via a controller $\pi_{t}$ computing $\ctrl(\x_t)$, scalar \rail enclosure nodes realizing $F(\x_t)$, and an affine layer computing $\x_{t+1} = \x_t + \bigl(F(\x_t) + B\,\pi_{t}(\x_t) + \ep_t\bigr)\,\delta$. \end{definition}

\begin{proposition}[Soundness of Bound Propagation on \nfg]
\label{prop:graph-soundness}
Let $\nfg$ have a \bart relaxation at each \rail node and a LiRPA bound at each nonlinear node in every $\pi_{(t)}$. Then any application of \rail to $\nfg$ over $\init \times \noise^\horizon$ produces linear bounds on $\state_t$ that overapproximate $\trajj^\dynsys(\init)_t$ for every $t \in [0..\horizon]$.
\end{proposition}
\subsection{Branch and Bound via \rail}
\label{sec:rail_bab}
We extend BaB to \nfg with \emph{enclosure refinement}: splitting a state node
$\x_t$ along a coordinate hyperplane, thereby refining every
enclosure that reads the split coordinate.

\begin{definition}[\rail Subproblem]
\label{def:rail-subproblem}
A \rail subproblem for $\nfg$ is a tuple $R = \langle C, \mathcal{H}, \mathcal{B}, \mathcal{N}\rangle$ consisting of an \emph{input cell} $C$ which is a hyperrectangle over the inputs
$(\x_0, \ep_0, \ldots, \ep_{\horizon-1})$ of \nfg; a set
$\mathcal{H}$ of \emph{state constraints} which are axis-aligned half-spaces of the form
$\{\x_t^k \leq \gamma\}$ or $\{\x_t^k \geq \gamma\}$, where $\x_t^k$ indicates the $k$th element of the state vector $\x_t$; a family $\mathcal{B} = \{B^{(i,t)}_\square\}$ of relaxed bounding sets, one for each
\rail enclosure node $f^i_t$ in $F(\x_t)$, each with domain induced by $C$ and
$\mathcal{H}$; and a set
$\mathcal{N}$ of \emph{pre-activation constraints}, one interval per split
neuron across the controllers $\pi_{0}, \ldots, \pi_{(\horizon-1)}$.
\end{definition}
\noindent
The subproblem $R$ covers the trajectories that satisfy its constraints:
inputs in $C$, states in every half-space of $\mathcal{H}$, and controller
pre-activations in every interval of $\mathcal{N}$. Constraints in
$\mathcal{H}$ on input nodes can be absorbed into $C$, since intersecting a
hyperrectangle with an axis-aligned half-space yields a hyperrectangle,
$\mathcal{H}$ thus retains only constraints on interior nodes. A subproblem
admits two split actions, each yeilding two \emph{child} subproblems:
\begin{itemize}
  \item \emph{enclosure refinement}: split a state node $\x_t$, input or
  interior, along a coordinate hyperplane $\{\x^k_t = \gamma\}$. One child
  receives the constraint $\{\x^k_t \leq \gamma\}$ and the other
  $\{\x_k^t \geq \gamma\}$. At $t = 0$, this is equivalent to partitioning $C$ along the hyperplane. Every enclosure that depends on $\x^k_t$ is refined along the grid hyperplane at $\gamma$ when $\gamma$ is one of its grid
  coordinates (\Cref{thm:refinement}), and otherwise \emph{snapped} outwards (yielding an overapproximation)
  to be refined along the nearest grid hyperplane bounding the child's
  constrained region. The controller $\pi_{t}$ also has its input domain   
  tightened.
  \item \emph{neuron split}: partition the pre-activation interval of a
  neuron in some $\pi_{t}$. Each child receives the
  corresponding subinterval in $\mathcal{N}$. For piecewise-linear
  activations, splitting an unstable neuron at zero fixes its phase.
\end{itemize}
Enclosure refinement generalizes input splitting from open-loop neural network verification. An open-loop input split bisects the input region along a chosen dimension (the $t = 0$ case of enclosure refinement). An interior split has no open-loop counterpart. It constrains an interior state node, propagating the constraint to every enclosure that depends on the split coordinate without further branching. Enclosure refinement integrates elements of input splitting and neuron splitting.
The operation is further complicated by grid alignment: refinement yields non-overlapping child domains only when the split hyperplane is a grid hyperplane, and the snapped domains overlap otherwise. Overlapping children have no analog in LiRPA-based BaB, where input split actions partition their domain.
\paragraph{Branching Heuristic}
Integrating enclosure refinements and neuron splits requires a branching heuristic that scores both action types on a common scale. Our heuristic, \muni (\textbf{M}ulti-action \textbf{U}nified \textbf{N}ode \textbf{I}mprovement), uses backward bound propagation to compute the gradient of the objective's bound with respect to each node, and takes the gradient magnitude as the estimated bound improvement of the split action the node admits (enclosure refinement at state and \rail nodes, neuron splits at pre-activation nodes). The split point is the midpoint of the relevant interval, snapped to the nearest grid hyperplane for \rail nodes. \muni generalizes BaBSR~\citep{bunel2020branch}, which scores neuron splits by cheap estimates of bound improvement and is detailed alongside other branching heuristics in Appendix~\ref{app:bab}. We leave more sophisticated heuristics to future work.

We call the resulting algorithm \clipper (\textbf{C}losed-\textbf{L}oop \textbf{I}terative \textbf{P}ropagation with \textbf{P}olyhedral \textbf{E}nclosure \textbf{R}efinement), and we show it in \Cref{alg:bab-rail}. Its soundness follows by induction on the subproblem tree. Refinements preserve the parent enclosure (\Cref{thm:refinement}), relaxations are sound for any slope (\Cref{thm:relaxation}), and neuron splits were previously shown to be sound~\citep{shi2025neural}. A parent is sound if all of its children are, since every parent trajectory lies in some child. The root covers every trajectory over the input cell, so returning \textsc{Verified} certifies the given specification.

\begin{algorithm}[t!]
\caption{\clipper: BaB on \rail}
\begin{algorithmic}[1]
\label{alg:bab-rail}
\REQUIRE Neural feedback graph \nfg, input cell $C_0$, property $\varphi$, branching heuristic \muni
\ENSURE \textsc{Verified} or counterexample $\state^*$
\STATE $\mathcal{B}_0 \gets$ \bart relaxations of the \rail nodes of \nfg
\STATE $\mathcal{R} \gets \{\langle C_0, \emptyset, \mathcal{B}_0, \emptyset \rangle\}$
\WHILE{$\mathcal{R} \neq \emptyset$}
    \STATE Select $R = \langle C, \mathcal{H}, \mathcal{B}, \mathcal{N} \rangle \in \mathcal{R}$ and remove it from $\mathcal{R}$
    \STATE $(\underline{y}, \overline{y}) \gets \textsc{railBound}(\nfg, R)$
    \IF{$[\underline{y}, \overline{y}]$ verifies $\varphi$}
        \STATE \textbf{continue} \COMMENT{$R$ is verified}
    \ELSIF{a counterexample $\state^* \in C$ is found}
        \RETURN $\state^*$
    \ELSE
        \STATE $a^* \gets \arg\max_{a \in \textsc{Actions}(R)} \muni(a, R)$
        \STATE $\mathcal{R} \gets \mathcal{R} \cup \textsc{Split}(R, a^*)$
    \ENDIF
\ENDWHILE
\RETURN \textsc{Verified}
\end{algorithmic}
\end{algorithm}

\section{Forward Reachability Analysis}
\label{sec:reachability-analysis}
Forward reachability analysis verifies a system $\dynsys$ by constructing a sequence of over-approximations $\hat{\State}_0, \ldots, \hat{\State}_\horizon$ with $\hat{\State}_t \supseteq \State_t$ for each $t \in [0..\horizon]$, where $\State_t = \trajj^\dynsys(\init)_t$. Safety is then checked against this sequence.

\begin{proposition}[Sufficient Condition for Safety]
\label{prop:reach-soundness}
Let $\hat{\State}_0, \ldots, \hat{\State}_\horizon$ satisfy $\hat{\State}_t \supseteq \State_t$ for all $t \in [0..\horizon]$. If $\hat{\State}_t \cap \avoid(t) = \emptyset$ for all $t$ and $\hat{\State}_{t^*} \subseteq \reach$ for some $t^* \in [0..\horizon]$, then $\dynsys$ is safe.
\end{proposition}

A reachability procedure returns \emph{verified} when it produces a sequence satisfying the conditions of \Cref{prop:reach-soundness}, and \emph{inconclusive} otherwise; because the check is against over-approximations, an inconclusive result does not imply that $\dynsys$ is unsafe. We define two procedures: \emph{concrete} reachability, the loose default, and \emph{symbolic} reachability, a tighter refinement. Both are evaluation strategies on $\nfg$, invoking \clipper as their bounding oracle.

\begin{table*}[!t]
\centering
\small
\renewcommand{\arraystretch}{1.2}
\setlength{\tabcolsep}{8pt}
\caption{Verification results on ARCH-COMP 2025 AINNCS benchmarks supported by \clipper. Instances are defined in Appendix~\ref{app:benchmarks}. $n$ is the state dimension, $\horizon$ the horizon. Cells report wall-clock seconds to certify, $\times$ denotes an inconclusive result (error or too imprecise), \texttt{U} an unsupported instance, and an \texttt{(F)} prefix a specification that does not hold, whose row instead reports set-computation time. Per instance the fastest tool is \textbf{bold} and the tightest (smallest terminal volume) \underline{underlined}. $\bigstar$ marks our method.}
\label{tab:results}
\begin{tabular}{l c c | c c c c c c}
\hline
Benchmark & $n$ & $\horizon$ & \clipper ($\bigstar$) & OvertPoly & OVERTVerify & immrax & CROWN-Reach & CORA \\
\hline
\begin{filecontents*}{data/results.csv}
benchmark,n,T,crest,overtpoly,overtverify,immrax,crownreach,cora
ACC,6,50,\textbf{\underline{0.5}},12.0,111.0,\underline{10.0},$\times$,60.4
TORA 1,4,20,\textbf{0.7},562.0,984.0,6.1,$\times$,\underline{2.0}
(F) TORA 2,4,10,\textbf{0.3},U,U,12.7,0.8,\underline{17.3}
(F) TORA 3,4,10,\textbf{0.1},U,U,8.1,0.7,\underline{1.4}
Unicycle,4,50,79.1,3941.0,16348.0,$\times$,$\times$,\textbf{\underline{13.3}}
(F) Single Pendulum,2,20,\textbf{0.2},1.1,0.7,3.8,1.1,\underline{0.7}
(F) Airplane,12,20,\textbf{14.6},U,U,\underline{61.1},$\times$,$\times$
Attitude Control,6,30,3.5,U,U,20.0,\textbf{\underline{2.4}},$\times$
Quadrotor,12,50,$\times$,U,U,$\times$,$\times$,$\times$
2D Spacecraft Docking,4,40,$\times$,U,U,$\times$,$\times$,$\times$
Navigation 1,4,30,\underline{865.1},U,U,$\times$,$\times$,\textbf{265.1}
Navigation 2,4,30,105.5,U,U,$\times$,$\times$,\textbf{\underline{7.4}}
\end{filecontents*}
\csvreader[late after line=\\]{data/results.csv}{}{%
\csvcoli & \csvcolii & \csvcoliii & \csvcoliv & \csvcolv & \csvcolvi & \csvcolvii & \csvcolviii & \csvcolix}%
\hline
\end{tabular}
\end{table*}

\subsubsection{Concrete Reachability}
\label{ssec:concrete-reachability}
Concrete reachability evaluates $\nfg$ in the manner of interval bound propagation (IBP) on the time axis. Each $\hat{\State}_t$ is a hyperrectangle, with $\hat{\State}_0$ the bounding box of $\init$. Given $\hat{\State}_t$, each component $\hat{\State}^i_{t+1}, i \in [n]$ of the next state set is bounded by invoking \clipper twice (once per bound direction) on the $i$-th component of $\next^\dynsys$ over $\x \in \hat{\State}_t$ and $\ep \in \noise$, returning an interval $[\underline{g}_i, \overline{g}_i]$. The product $\hat{\State}_{t+1} = \prod_{i=1}^n
[\underline{g}_i, \overline{g}_i]$ is a hyperrectangle and feeds the next step. Cross-coordinate and cross-time correlations are discarded at each step.

\subsubsection{Symbolic Reachability}
\label{ssec:symbolic-reachability}
Symbolic reachability evaluates $\nfg$ by propagating linear bounds in $(\x_0, \ep_0, \ldots, \ep_{\horizon-1})$ across each time step, preserving the cross-time and cross-coordinate correlations that concrete reachability discards. Each per-step subgraph contributes \bart relaxations at the \rail nodes representing $F$, sized from the concrete over-approximations $\hat{\State}_t$, and LiRPA relaxations at the nonlinear nodes in $\pi_{t}$, sized from IBP bounds within the controller.

This scaffolding is common to both traditions \clipper unifies. Within a feedforward network, IBP bounds provide the domains over which LiRPA bounds are defined~\cite{xu2020automatic}, whereas in combinatorial NFS solvers, per-step concrete reachable sets are the abstractions over which the symbolic problem is defined~\cite{sidrane2022overt,akinwande2026polyhedralenclosuresefficientcombinatorial}. Symbolic reachability via \clipper inherits both: a single concrete pre-pass sizes the \bart relaxations of the dynamics and the LiRPA relaxations of the controller across the horizon of $\nfg$.

When \clipper refines a step, the scaffolding must be refined as well. Splitting the state node at step $t^*$ partitions $\hat{\State}_{t^*}$ into cells. The \bart relaxation on a cell $S$ over-approximates $F$ only over $S$, so the bound it induces at a later step is sound only for trajectories whose state at $t^*$ lies in $S$. Each cell therefore receives its own scaffolding, sized by concrete reachability restarted from $S$ rather than by the parent sets $\hat{\State}_t$, and a backward pass over the remaining horizon yields its bounds. Each cell is checked against the specification independently, and the check succeeds when every cell passes. This is \clipper's branch-and-bound applied to reachability analysis, and the cells are the children of an enclosure refinement at $t^*$.

Concrete reachability is the default. However, because it discards cross-time correlations, it may fail to certify challenging reach-avoid specifications. Symbolic reachability then tightens the specific $\hat{\State}_t$ at which the check fails.

\section{Evaluation}
\label{sec:evaluation}
We evaluate \clipper on the ARCH-COMP 2025 AINNCS suite~\citep{manzanas_lopez2025arch}, which comprises 12 reach-avoid instances across nine neural feedback systems (two systems contribute multiple instances). Following the guidelines of the organizers, we discretize the continuous-time plants using forward Euler integration. On AINNCS we compare against the 2025 participants that permit discrete-time analysis, either natively (CORA) or via a fair modification (CROWN-Reach, immrax; \Cref{app:baselines}), together with the combinatorial solvers OvertPoly and OVERTVerify. Continuous-time performance does not automatically transfer to the discrete-time setting, as the resulting closed-loop behavior differs.
Details about the AINNCS benchmarks are provided in Appendix~\ref{app:experimental-setup}.

\begin{figure*}[!t]
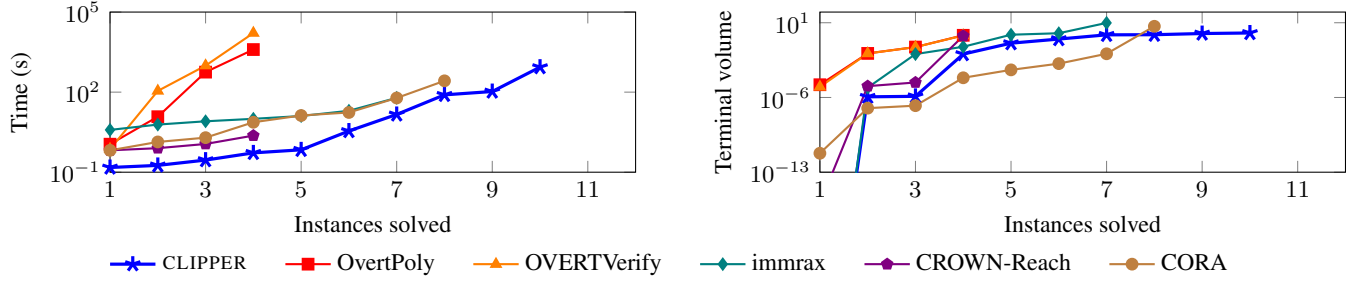

\centering
\begin{tikzpicture}
\begin{axis}[
    width=0.48\textwidth, height=3.7cm,
    xlabel={Instances solved}, ylabel={Time (s)},
    ymode=log, xmin=1, xmax=12, ymin=0.1, ymax=100000,
    xtick={1,3,5,7,9,11},
    legend columns=-1, legend to name=cactuslegend,
    legend style={draw=none, font=\footnotesize,
                  /tikz/every even column/.append style={column sep=1.2em}},
    tick label style={font=\footnotesize}, label style={font=\footnotesize},
    unbounded coords=jump,
]
\input{data/cactus_time_plots.tex}
\end{axis}
\end{tikzpicture}
\hfill
\begin{tikzpicture}
\begin{axis}[
    width=0.48\textwidth, height=3.7cm,
    xlabel={Instances solved}, ylabel={Terminal volume},
    ymode=log, xmin=1, xmax=12, ymin=1e-13, ymax=100,
    xtick={1,3,5,7,9,11},
    tick label style={font=\footnotesize}, label style={font=\footnotesize},
    unbounded coords=jump,
]
\input{data/cactus_volume_plots.tex}
\end{axis}
\end{tikzpicture}

\pgfplotslegendfromname{cactuslegend}
\vspace*{0.6em}
\caption{Cactus plots over the 12 AINNCS instances. An instance is \emph{solved} when a tool produces its reachable-set result (a certification, or a set computation on falsifiable instances), and each tool's solved instances are sorted by the plotted quantity. A point $(x,y)$ means $x$ instances were solved within $y$, so lower and farther right is better. \textbf{Left:} per-instance time. \textbf{Right:} terminal reachable-set volume (tightness).}
\label{fig:cactus}
\end{figure*}

\paragraph{Falsifiable AINNCS specifications.} A few AINNCS specifications do not hold. We cross-check this with one sampler over $\init$ (fixed seed and budget) applied identically across tools: a sampled trajectory that intersects $\avoid(t)$ at some $t$ witnesses \eqref{eq:avoid_prop} failing, and one that lies outside $\reach$ at every $t$ suggests \eqref{eq:reach_prop} failing. We report set-computation time on these instances as for the others to measure scalability and tightness on identical dynamics. Such rows are marked \texttt{(F)}.

\paragraph{Coverage.} \clipper solves 10 of the 12 instances (\Cref{tab:results}), certifying six of the eight instances whose specifications hold, and computing reachable sets on all four that do not. No baseline solves an instance that \clipper does not.
CORA is the closest, solving eight. Its strength in the continuous-time setting carries over here, which makes it the state-of-the-art baseline for discrete-time reachability.
The next best is immrax with seven instances, and OvertPoly, OVERTVerify, and CROWN-Reach solve four each.
The combinatorial solvers account for most of the coverage gap, as OvertPoly and OVERTVerify support only 4 of the 12 instances. Their MILP encodings require \ReLU controllers, which rules out seven of the eight instances they leave unsupported (\Cref{app:benchmarks}). \clipper inherits their grid-based enclosure of the dynamics but places no restriction on the controller's activations, and therefore supports all 12 instances. These tools are the closest antecedents to \clipper, and the coverage gain over them is substantial.

\paragraph{Tightness.} Tools that unify dynamics and controller abstractions (CORA and \clipper) lead on both tightness and coverage. CORA attains the smallest terminal volume on six instances and \clipper on two, leaving one each to immrax and CROWN-Reach. The volume cactus plot (\Cref{fig:cactus}, right) sharpens the split: CORA's points lie below \clipper's on most instances, while \clipper's extends two instances further. 
immrax solves Airplane and Attitude Control, which CORA leaves inconclusive, but takes four to six times longer on them than \clipper.
CORA's unified abstraction does not scale with dimension. The two instances \clipper solves and CORA does not, Airplane ($n = 12$) and Attitude Control ($n = 6$), are among the largest in the suite, and we attribute this to a control-native abstraction whose cost grows with the state dimension and the network size.
Conversely, \clipper reports each $\hat{\State}_t$ as an axis-aligned hyperrectangle, which is less expressive than CORA's zonotopes and accumulates wrapping error when the reachable set rotates. This is the likely reason CORA outperforms \clipper on Unicycle and the Navigation instances, whose vector fields rotate the set as the heading turns. \clipper compensates by refining more, which costs time. We leave the exploration of more expressive abstractions to future work.

\paragraph{Time.} \clipper is fastest on most of the instances it solves, and its time cactus plot reflects this (\Cref{fig:cactus}, left). Against the combinatorial solvers, \clipper is faster on every instance they support, by factors between $3.5$ and $1000$. Against the other participants, the margin is smaller and not uniform. CORA is faster on the Unicycle and the Navigation instances, and CROWN-Reach edges \clipper out on Attitude Control.

\begin{table}[h]
\centering
\footnotesize
\caption{Relaxation ablation on \textbf{Unicycle} at $t^\ast = 15$. \emph{Volume} is the reachable-set volume at $t^\ast$; \emph{Time} is the wall-clock time to compute and optimize the relaxation.}
\label{tab:relaxation}
\setlength{\tabcolsep}{4pt}
\begin{tabular}{l c c c}
\hline
 & \bart ($\bigstar$) & \bart$+$OVERT & \bart$+$CROWN \\
\hline
\begin{filecontents*}{data/relaxation.csv}
metric,bart,overt,abcrown
Volume at $t^\ast$,4.07E-06,1.23E-06,1.21E-06
Time (s),12.4,88.3,88.0
\end{filecontents*}
\csvreader[late after line=\\]{data/relaxation.csv}{}{%
\csvcoli & \csvcolii & \csvcoliii & \csvcoliv}%
\hline
\end{tabular}
\end{table}
\subsection{Ablations}
On the \textbf{Unicycle} benchmark, whose vector field couples speed and heading through $x_4 \cos x_3$ and $x_4 \sin x_3$ (Appendix~\ref{app:benchmarks}), we compare our \bart approach against two alternatives that decompose each coupled term into primitives, relax each independently, and compose by backward substitution, as \nfg would if the dynamics used LiRPA's native operators. \bart$+$\emph{OVERT}~\citep{sidrane2022overt} bounds each primitive with the enclosure's grid-derived bound, while \bart$+$CROWN uses LiRPA native relaxations~\citep{shi2025neural}. All three optimize their slopes for the same 50 gradient steps and none refines or splits, so the comparison isolates relaxation quality (\Cref{tab:relaxation}). At that fixed budget, the alternative approaches are $3.3\times$ tighter at $t^\ast$ but $7\times$ slower; the \bart baseline stays within the same order of magnitude at a fraction of the cost. This is consistent with the pattern of enclosure composition being computationally efficient at the cost of precision.

We then ablate \clipper's branching actions (\Cref{tab:ablation}). \emph{Input splitting} restricts refinement to $t = 0$, the open-loop default, while \emph{enclosure refinement} also splits interior state nodes, refining every enclosure that reads them. Both branch until 96 children, isolating where the splits fall. At that fixed budget, enclosure refinement returns a $1.4\times$ to $1.7\times$ tighter terminal set on both benchmarks.\begin{table}[t]
\centering
\footnotesize
\caption{Branching ablation on a benchmark subset. \emph{Time} is wall-clock seconds and \emph{Volume} the terminal reachable-set volume.
$\bigstar$ marks our method.}
\label{tab:ablation}
\setlength{\tabcolsep}{3pt}
\begin{tabular}{l c c c c}
\hline
 & \multicolumn{2}{c}{Enclosure refinement ($\bigstar$)} & \multicolumn{2}{c}{Input splitting} \\
\cline{2-3} \cline{4-5}
Benchmark & Time & Volume & Time & Volume \\
\hline
\begin{filecontents*}{data/ablation.csv}
benchmark,refine_only_time,refine_only_volume,split_only_time,split_only_volume
Unicycle,107.9,5.40E-06,86.1,7.39E-06
Attitude Control,2619.0,2.71E-07,2272.0,4.53E-07
\end{filecontents*}
\csvreader[late after line=\\]{data/ablation.csv}{}{%
\csvcoli & \csvcolii & \csvcoliii & \csvcoliv & \csvcolv}%
\hline
\end{tabular}
\end{table}
\section{Conclusions and Future Work}
\rail exposes polyhedral enclosures of nonlinear dynamics to LiRPA, and \clipper branches over enclosure refinements and neuron splits on the resulting graph, bringing parallelizable bound propagation to neural feedback systems. On ARCH-COMP 2025 AINNCS, \clipper solves more instances than the discrete-time state of the art.
\clipper's main limitation is that it reports each reachable set as axis-aligned hyperrectangles, losing tightness to more expressive representations. Richer representations and stronger branching heuristics are the natural next steps.

\clearpage
\makeatletter\global\vsize\textheight\global\@colroom\textheight\makeatother
\bibliography{refs}
\clearpage
\appendix
\label{page:appendix}
\etocsettocstyle{\section*{Appendix Contents}}{}
\renewcommand{\thesubparagraph}{\thesubsection.\arabic{subparagraph}}
\makeatletter
\def\addcontentsline#1#2#3{%
  \addtocontents{#1}{\protect\contentsline{#2}{#3}{\thepage}{}}%
}
\makeatother
\setcounter{tocdepth}{3}
\setcounter{secnumdepth}{3}
\localtableofcontents
\section{Formal Definitions}
\label{app:formal-defs}
\subsection{Polyhedra}
\label{app:geometric-defs}
Let $\pset = \{\vec{p}_0, \dots, \vec{p}_k\} \subset \reals^n$. The \emph{convex hull} of $\pset$ is
\begin{equation}
    \conv(\pset)
    = \{\theta_0 \vec{p}_0 + \ldots + \theta_k \vec{p}_k \mid \vec{\theta} \cdot \vec{1} = 1,\ \vec{\theta} \vgeq \vec{0}\},
\end{equation}
and the \emph{affine hull} $\aff(\pset)$ is defined identically but without the constraint $\vec{\theta} \vgeq \vec{0}$. $\pset$ is \emph{affinely independent} iff $\{\vec{p}_1 - \vec{p}_0, \dots, \vec{p}_k - \vec{p}_0\}$ is linearly independent, equivalently iff $\dim \aff(\pset) = k$. In general $\dim \aff(\pset) \leq k$. The \emph{polyhedron formed by} $\pset$ is $\conv(\pset)$. A subset of $\reals^n$ is a \emph{polyhedron} if it is the convex hull of a finite set of points in $\reals^n$.

\appsubsubsection{$k$-Simplex}
\begin{definition}
    A polyhedron $\simp$ is a \emph{$k$-simplex} if it is the convex hull of $k+1$ affinely independent points. A polyhedron is a \emph{simplex} if it is a $k$-simplex for some $k$, and $k$ is called its \emph{dimension}.
\end{definition}
\noindent
If $\simp$ is a $k$-simplex, its \emph{vertices}, denoted $\svert(\simp)$, are the unique set $\pset$ of $k+1$ points with $\simp = \conv(\pset)$. A \emph{face} of $\simp$ is the convex hull of any non-empty subset of $\svert(\simp)$.

\appsubsubsection{Simplicial Complex}
\begin{definition}
A \emph{simplicial complex} $\scomp$ is a finite set of simplices such that:
\begin{itemize}
    \item every face of a simplex in $\scomp$ is also in $\scomp$;
    \item every non-empty intersection of two simplices $\simp_1, \simp_2 \in \scomp$ is a face of both $\simp_1$ and $\simp_2$.
\end{itemize}
\end{definition}
\noindent
The \emph{dimension} of $\scomp$ is the largest dimension of any simplex it contains. A simplicial complex of dimension $k$ is a \emph{pure simplicial $k$-complex} if every simplex in $\scomp$ is a face of some $k$-dimensional simplex in $\scomp$.

\appsubsubsection{Point Set Triangulation}
Let $\pset$ be a finite set of points in $\reals^n$. We say $\pset$ is \emph{full-dimensional} if it contains $n+1$ affinely independent points.
\begin{definition}
If $\pset$ is a finite, full-dimensional set of points in $\reals^n$, then a pure simplicial $n$-complex $\scomp$ is a \emph{triangulation of $\pset$} if $\pset = \bigcup_{\simp \in \scomp} \svert(\simp)$ and $\conv(\pset) = \bigcup_{\simp \in \scomp} \simp$.
\end{definition}
\noindent
Let $C$ be a closed $n$-ball, we write $C^O$ for the corresponding open $n$-ball and $C^S$ for the hypersphere forming its surface. For a set of points $\pset$, let $V_\pset(C) = C \cap \pset$. For a polyhedron $\poly$, the \emph{circumsphere} of $\poly$ (when it exists) is the hypersphere passing through every vertex of $\poly$. Circumspheres always exist for simplices and for hyperrectangles.

Let $\pset$ be a finite, full-dimensional set of points in $\reals^n$ and let $\scomp$ be a triangulation of $\pset$. For each $n$-simplex $\simp \in \scomp$, let $C(\simp)$ denote the $n$-ball whose boundary is the circumsphere of $\simp$. We say $\simp$ satisfies the \emph{Delaunay condition}, and call it a \emph{Delaunay simplex of $\pset$}, if $V_\pset(C(\simp)^O) = \emptyset$ (equivalently, the only points of $\pset$ in $C(\simp)$ lie on its surface). $\scomp$ is a \emph{Delaunay triangulation} if every $n$-simplex in $\scomp$ satisfies the Delaunay condition. 

For $\vec{x} \in \conv(\pset)$, let $\simp_\scomp(\vec{x})$ denote the $n$-simplex of $\scomp$ containing $\vec{x}$. If $\vec{x}$ lies on a face of dimension less than $n$ and is therefore contained in multiple $n$-simplices, $\simp_\scomp(\vec{x})$ chooses one in a deterministic way. Writing $\svert(\simp_\scomp(\vec{x})) = \mathbf{p}$ in some fixed order, $ \theta = \theta_\scomp(\vec{x}) \in \reals^{n+1}$ is the unique convex combination vector with $\vec{x} = \theta \cdot \mathbf{p}$.


\appsubsubsection{Grids}
\begin{definition}
A \emph{grid of dimension $n$} is the Cartesian product $G = G_1 \times \dots \times G_n$, where each $G_i \subset \reals$ is a finite set with at least two elements. The \emph{domain} of $G$ is $\dom(G) = \conv(G)$.
\end{definition}
\noindent
A \emph{grid cell} of $G$ is the vertex set $\State$ of an $n$-dimensional hyperrectangle $R \subseteq \dom(G)$ satisfying $G \cap R = \State$; equivalently, the $2^n$ vertices of $R$ all lie in $G$, and no other point of $G$ lies in $R$. For $\state \in G$, we write
$\cells(G, \state) = \{\State \mid \State \text{ is a grid cell of } G,\ \state \in \State\}.$

\appsubsubsection{Grid Hyperplane}
\label{app:grid_hyperplane}
\begin{definition}
Let $\pset = G_1 \times \cdots \times G_n$ be a grid in $\reals^n$. A \emph{grid hyperplane} of $\pset$ is a hyperplane of the form
\[
H := \{\x \in \reals^n : x_i = c\},
\]
for some axis index $i \in \{1, \ldots, n\}$ and $c \in G_i$ with $\min G_i < c < \max G_i$.
\end{definition}

\subsection{Polyhedral Enclosures}
\label{app:poly_encs}
This section reproduces the formal machinery for polyhedral enclosures from the OvertPoly formulation cited in the main paper, we refer the reader there for full definitions and proofs.
\appsubsubsection{Bounding Sets}
\begin{definition}
A \emph{bounding set} is a tuple $\bset = \langle n, \pset, L, U \rangle$ where $n \in \nats$, $\pset$ is a grid of dimension $n$, and $L, U : \pset \to \reals$ satisfy $L(\vec{p}) \leq U(\vec{p})$ for every $\vec{p} \in \pset$. The \emph{domain} of $\bset$ is $\dom(\bset) = \conv(\pset)$.
\end{definition}

\appsubsubsection{Polyhedra formed by Bounding Sets}
\begin{definition}
Let $\bset = \langle n, \pset, L, U \rangle$ be a bounding set. The \emph{lifted vertex set} of $\bset$ is
\[
    V(\bset) := \{(\vec{p}, L(\vec{p})) \mid \vec{p} \in \pset\} \cup \{(\vec{p}, U(\vec{p})) \mid \vec{p} \in \pset\}
\]
where $V(\bset) \subset \reals^{n+1}$. The \emph{polyhedron formed by $\bset$} is
\[
    \poly(\bset) := \conv(V(\bset)).
\]
\end{definition}

\appsubsubsection{Polyhedral Enclosures}
\begin{definition}
Let $\bset = \langle n, \pset, L, U \rangle$ be a bounding set, let $\Delta$ be a Delaunay triangulation of $\pset$, and let $\Delta_n$ denote the set of $n$-simplices in $\Delta$. For each $\simp \in \Delta_n$, the \emph{restriction of $\bset$ to $\simp$} is the bounding set
\[
    \bset|_\simp := \langle n, \svert(\simp), L|_{\svert(\simp)}, U|_{\svert(\simp)} \rangle.
\]
The \emph{polyhedral enclosure formed by $\bset$ and $\Delta$} is
\[
    \enc(\bset, \Delta) := \bigcup_{\simp \in \Delta_n} \poly(\bset|_\simp).
\]
\end{definition}
\noindent
The \emph{lower} and \emph{upper surfaces} of $\enc(\bset, \Delta)$ are the functions $\Ltri, \Utri : \dom(\bset) \to \reals$ given by
\begin{subequations}
\label{eq:surface_defs}
\begin{align}
    \Ltri(\state) &:= \min\{y \in \reals \mid (\state, y) \in \enc(\bset, \Delta)\}, \label{eq:lower_surface} \\
    \Utri(\state) &:= \max\{y \in \reals \mid (\state, y) \in \enc(\bset, \Delta)\}. \label{eq:upper_surface}
\end{align}
\end{subequations}

\appsubsubsection{Function Enclosures}
\begin{definition}
\label{def:funcEnc}
Let $\bset = \langle n, \pset, L, U \rangle$ be a bounding set. A function $f : D \to \reals$ with $\dom(\bset) \subseteq D \subseteq \reals^n$ is \emph{enclosed} by $\bset$ if $(\state, f(\state)) \in \enc(\bset, \Delta)$ for every $\state \in \dom(\bset)$ and every Delaunay triangulation $\Delta$ of $\pset$.
\end{definition}
\noindent
Equivalently, $f$ is enclosed by $\bset$ if and only if $\Ltri(\state) \leq f(\state) \leq \Utri(\state)$ for every $\state \in \dom(\bset)$ and every Delaunay triangulation $\Delta$ of $\pset$.

\subsection{LiRPA Algorithms}
\label{app:lirpa}
This section reproduces the formal machinery for linear relaxation based perturbation analysis (LiRPA) from~\cite{xu2020automatic}, we refer the reader there for full definitions and proofs.

\appsubsubsection{Neural Network Graphs}
\label{app:nn_graph}

\begin{definition}
\label{def:nn_graph}
A \emph{neural network graph} is a tuple $G = (V, E, \{h_i\}_{i \in V})$ where $(V,E)$ is a directed acyclic graph with topologically ordered node set $V = \{1, \ldots, n\}$ and edge set $E \subseteq V \times V$, where $(i,j) \in E$ denotes that the value at $i$ is an input to $j$. For each $i \in V$, write $u(i) = \{j : (j,i) \in E\}$ for the predecessors of $i$ and $d_i \in \nats$ for the output dimension of $i$. Each node carries a (possibly nonlinear) function $h_i$ taking the concatenated predecessor values as input and producing an output in $\reals^{d_i}$. We single out node $1$ as the input node, with $u(1) = \emptyset$ and $h_1(\state) = \state$, and node $n$ as the output node, so that evaluating the graph in topological order on input $\state \in \netDom$ yields $\pi(\state) = h_n$.
\end{definition}

\appsubsubsection{Linear Relaxations of Nodes}
LiRPA decomposes the global bounding problem into per-node \emph{local} linear relaxations, which are then composed across the graph.
\begin{definition}[Local Linear Relaxation]
\label{def:local_relax}
Fix a node $i \in V$ with predecessor values constrained to a hyperrectangle $\mathbf{v} \in [\underline{\mathbf{v}}_{u(i)}, \overline{\mathbf{v}}_{u(i)}]$. A \emph{linear relaxation} of $h_i$ on this domain is a tuple $(\underline{A}_i, \underline{c}_i, \overline{A}_i, \overline{c}_i)$, with matrices $\underline{A}_i, \overline{A}_i$ and vectors $\underline{c}_i, \overline{c}_i$ of dimensions matching the input and output of $h_i$, satisfying, componentwise,
\[
\forall \mathbf{v} \in [\underline{\mathbf{v}}_{u(i)}, \overline{\mathbf{v}}_{u(i)}].\quad
\underline{A}_i \mathbf{v} + \underline{c}_i
\;\leq\; h_i(\mathbf{v})
\;\leq\; \overline{A}_i \mathbf{v} + \overline{c}_i.
\]
For affine $h_i(\mathbf{v}) = W\mathbf{v} + b$ the relaxation is \emph{exact}: $\underline{A}_i = \overline{A}_i = W$ and $\underline{c}_i = \overline{c}_i = b$. For nonlinear $h_i$ (e.g.\ \ReLU, sigmoid, tanh) the relaxation is sound but inexact, and is typically chosen to minimize area between the bounding hyperplanes given the interval $[\underline{\mathbf{v}}_{u(i)}, \overline{\mathbf{v}}_{u(i)}]$.
\end{definition}

\begin{proposition}[Bound Propagation~\citep{xu2020automatic}]
\label{prop:lirpa}
Given a neural network graph $\pi$, with neural graph $G$, an input set $\State \subseteq \netDom$, and a local linear relaxation at each nonlinear node, LiRPA computes matrices $\underline{\mathbf{W}}_o, \overline{\mathbf{W}}_o$ and vectors $\underline{\mathbf{b}}_o, \overline{\mathbf{b}}_o$ such that
\[
\forall \state \in \State.\quad
\underline{\mathbf{W}}_o \state + \underline{\mathbf{b}}_o
\;\leq\; \pi(\state)
\;\leq\; \overline{\mathbf{W}}_o \state + \overline{\mathbf{b}}_o.
\]
The bounds are computed by forward or backward propagation through the graph and admit efficient GPU implementation.
\end{proposition}
We refer the reader to~\citep{xu2020automatic} for details.
\appsubsubsection{Concretization on $\ell_p$ Input Sets}

When $\State = \{\state : \|\state - \state_0\|_p \leq \epsilon\}$ is an $\ell_p$-ball, scalar (componentwise) bounds on $\pi(\state)$ follow from the linear bounds via Hölder's inequality:
\[
\max_{\state \in \State} \overline{\mathbf{W}}_o \state + \overline{\mathbf{b}}_o
\;=\; \overline{\mathbf{W}}_o \state_0 + \epsilon \,\|\overline{\mathbf{W}}_o\|_{q,\text{row}} + \overline{\mathbf{b}}_o,
\quad \tfrac{1}{p} + \tfrac{1}{q} = 1,
\]
where $\|\cdot\|_{q,\text{row}}$ denotes row-wise dual norm. This avoids any explicit optimization over $\State$.

\appsubsubsection{Optimizable Relaxations ($\alpha$-CROWN)}

For piecewise-linear activations such as \ReLU, the lower relaxation admits a free slope parameter per unstable neuron, which we collect into a vector $\alpha \in [0,1]^k$. Each choice of $\alpha$ yields a sound bound. Tighter bounds are obtained by maximizing the lower bound (resp.\ minimizing the upper bound) over $\alpha$ via projected gradient ascent/descent. The resulting tightened bounds are competitive with LP-relaxation-based verifiers while admitting efficient GPU implementation. We refer to this scheme as $\alpha$-CROWN.

\subsection{Branching for Neural Network Verification}
\label{app:bab}
This section provides background on branch-and-bound (BaB) for neural network verification. We refer the reader to~\citep{bunel2020branch, de2021improved, wang2021beta} for full treatments.

The bounds produced by LiRPA (\Cref{prop:lirpa}) are sound but inexact, and may be too loose to certify a property of interest over $\State$. Branch-and-bound addresses this by partitioning $\State$ (or the network's internal activations) into subdomains on which LiRPA produces tighter bounds, and recursing until each subdomain is either verified or a counterexample is found.

\begin{definition}[Branch-and-Bound for NN Verification]
\label{def:bab}
Given a network $\pi$, an input set $\State$, and a property
$\varphi$ to verify, BaB maintains a set of subproblems
$\mathcal{D} = \{D_1, D_2, \ldots\}$ initialized with
$\mathcal{D} = \{\langle \State, \emptyset \rangle\}$, where each
subproblem $D = \langle X, \mathcal{N} \rangle$ pairs a subset
$X \subseteq \State$ with a set $\mathcal{N}$ of fixed activation
patterns. At each iteration, BaB selects a subproblem
$D \in \mathcal{D}$, computes bounds on $\pi$ over $D$ using a sound
bounding procedure (e.g.\ LiRPA), and either (i) resolves $D$ if
the bounds verify or falsify $\varphi$, or (ii) splits $D$ into
subdomains via a \emph{branching heuristic} and adds them to
$\mathcal{D}$. The procedure terminates when $\mathcal{D}$ is empty
(verified) or a counterexample is found (falsified).
\end{definition}

Splits fall into two families. \emph{Input splits} partition the input set $\State$ directly, typically by bisecting along a chosen input dimension. \emph{Neuron splits} fix the activation pattern of a chosen ReLU (i.e., constrain its pre-activation to be either nonnegative or nonpositive), which tightens the network's relaxation on that subdomain. The choice of which dimension or neuron to split on is the branching heuristic.

\appsubsubsection{Sound Neuron Splits ($\beta$-CROWN)}

A neuron split imposes a linear constraint on a pre-activation, which a naive LiRPA pass cannot exploit since it bounds each neuron independently. $\beta$-CROWN~\citep{wang2021beta} extends LiRPA with per-neuron split constraints using Lagrangians. Each split constraint contributes a multiplier $\beta \geq 0$ that modifies the backward propagation rules at the constrained neuron. For any choice of $\beta \geq 0$, the resulting bounds are sound, but tighter bounds can be obtained by jointly optimizing $\alpha$ and $\beta$ using projected gradient methods in a manner analogous to $\alpha$-CROWN. The procedure can be run in parallel across subdomains in $\mathcal{D}$, which makes GPU enabled BaB with $\beta$-CROWN as the bounding oracle the basis of state-of-the-art complete verifiers.

\appsubsubsection{Input Split Heuristics}

\paragraph{BaBLL (Longest Length, \cite{bunel2020branch}).} The simplest input splitting heuristic is to bisect along the input dimension with the largest current width, $\arg\max_j (\overline{x}_j - \underline{x}_j)$, where $[\underline{x}, \overline{x}]$ is the bounding box of the current subdomain. This heuristic is cheap to compute and dimension-agnostic, but ignores how the network depends on each input.

\paragraph{BaBSB (Smart Branching, \cite{bunel2020branch}).} This heuristic scores each input dimension $j$ using cheap estimates of expected bound improvement from splitting along $j$, and branches on the dimension with the highest score. The estimate typically uses the bound coefficients already computed by LiRPA on the current subdomain, thereby weighting each dimension by its sensitivity in the current bound.

\appsubsubsection{Neuron Split Heuristics}

\paragraph{BaBSR (Smart ReLU~\citep{bunel2020branch}).} Similar to BaBSB, this heuristic assigns each unstable ReLU a score based on a cheap estimate of the objective improvement after splitting it, and selects the neuron with the highest score. BaBSR is cheap but its scores are coarse approximations of the true post-split improvement.

\paragraph{FSB (Filtered Smart Branching)~\citep{de2021improved}.} This heuristic was designed to mimic strong branching on a filtered shortlist. It uses BaBSR to propose a small set of promising candidates, and computes bounds (using a LiRPA pass) to select the candidate yielding the greatest bound improvement. FSB produces higher-quality splits than BaBSR at greater per-step cost.

\paragraph{k-FSB.} This heuristic parametrizes FSB using the shortlist size $k$ 
As implemented in $\alpha,\beta$-CROWN~\citep{wang2021beta}, BaBSR proposes the top $k$ candidates and an FSB-style heuristic selects among them. 
Smaller values of $k$ recover behavior closer to BaBSR, while larger values of $k$ approach full FSB. In practice $k$-FSB with small $k$ trades a small amount of branching quality for substantially lower per-step cost, and is the default branching heuristic in $\alpha,\beta$-CROWN.

\paragraph{BBPS (Propagation with Shortcuts)~\citep{shi2025neural}.} This heuristic generalizes FSB-style scoring to non-ReLU nonlinearities. Rather than relying on ReLU-specific bound estimates, it uses linear bounds propagated from the candidate split node directly to the input as a shortcut for estimating post-split improvement, thereby making branching tractable for networks with general nonlinearities (Sigmoid, Tanh, GeLU, etc.).

\newpage
\section{Auxiliary Results}
\label{app:aux_results}
\subsection{Triangulating Grids}
\label{app:triangulating_grids}
\appsubsubsection{Delaunay Cell Containment}
\label{app:delaunay_cells}
\begin{lemma}
\label{lem:delaunay_cells}
Let $\pset$ be a grid in $\reals^n$ and let $\Delta$ be a Delaunay triangulation of $\pset$. Every $n$-simplex of $\Delta$ is contained in some grid cell of $\pset$.
\end{lemma}

\begin{proof}
Let $\simp \in \Delta$ be an $n$-simplex with circumcenter $c \in \reals^n$ and circumradius $r$. Its vertices lie at distance $r$ from $c$, and by the Delaunay empty-circumsphere property, no point of $\pset$ lies at distance less than $r$ from $c$. Hence the vertices of $\simp$ are contained in $N := \arg\min_{p \in \pset} \|p - c\|$.

Since $\pset = G_1 \times \cdots \times G_n$, the squared distance $\|p - c\|^2 = \sum_i (p_i - c_i)^2$ separates by coordinate, so $N = N_1 \times \cdots \times N_n$ where $N_i := \arg\min_{q \in G_i} (q - c_i)^2$. Each $N_i$ has size $1$ or $2$, with size $2$ exactly when $c_i$ is the midpoint of two adjacent elements of $G_i$. Let $k$ be the number of coordinates $i$ for which $|N_i| = 2$. 
Then $|N| = 2^k$ and $\dim \aff(N) = k$, with $k \leq n$ since there are only $n$ coordinates.

The $n+1$ affinely independent vertices of $\simp$ lie in $N$, so $\dim \aff(\svert(\simp)) = n \leq \dim \aff(N) = k$, hence $n = k$. Thus $c$ is the midpoint of an adjacent pair in every coordinate. That is, $c$ is the center of some grid cell $\cell$, and $N$ is exactly the $2^n$ corners of $\cell$. Therefore $\simp \subseteq \cell$.
\end{proof}

\appsubsubsection{No Crossing Grid Hyperplanes}
\label{app:no_crossing}
\begin{corollary}
\label{cor:no_crossing}
Let $\pset$ be a grid in $\reals^n$, let $H$ be a grid hyperplane of $\pset$, and let $\Delta$ be a Delaunay triangulation of $\pset$. Each $n$-simplex of $\Delta$ is contained in one of the closed half-spaces bounded by $H$.
\end{corollary}

\begin{proof}
By Lemma \ref{lem:delaunay_cells}, every $n$-simplex of $\Delta$ is contained in a single grid cell. Each grid cell lies in one of the two closed half-spaces bounded by $H$, so the simplex does too, and hence does not cross $H$.
\end{proof}

\appsubsubsection{Splitting Delaunay Triangulations}
\label{app:delaunay_split}
\begin{lemma}
\label{lem:delaunay_split}
Let $\pset$ be a grid in $\reals^n$, let $H$ be a grid hyperplane of $\pset$, and let $\pset_1, \pset_2$ be the sub-grids of $\pset$ on the two closed sides of $H$. Let $\Delta$ be a Delaunay triangulation of $\pset$, and let $\Delta_j := \{\simp \in \Delta \mid \simp \subseteq \conv(\pset_j)\}$ for $j \in \{1, 2\}$. Then $\Delta_1$ and $\Delta_2$ are Delaunay triangulations of $\pset_1$ and $\pset_2$, respectively.
\end{lemma}

\begin{proof}
By definition, $\pset_j = \pset \cap \overline{H}_j$ for the closed half-spaces $\overline{H}_1, \overline{H}_2$ bounded by $H$. Consequently $\conv(\pset_j)$ is the union of the grid cells of $\pset$ contained in $\overline{H}_j$.

Fix a grid cell $\cell \subseteq \overline{H}_j$, and let $\Delta_\cell := \{\simp \in \Delta : \simp \subseteq \cell\}$. Since $\Delta$ triangulates $\conv(\pset) \supseteq \cell$, the cell $\cell$ is covered by simplices of $\Delta$. By Lemma \ref{lem:delaunay_cells}, each such simplex lies in a single grid cell, and since grid cells have pairwise-disjoint interiors, any simplex with interior points in $\cell$ must be contained in $\cell$. Hence the simplices covering the interior of \cell all belong to $\Delta_\cell$, and $\bigcup_{\simp \in \Delta_\cell} \simp$ is a closed set containing the interior of \cell, hence containing $\cell$.
The corners of $\cell$ lie in $\pset_j$, so $\Delta_\cell \subseteq \Delta_j$, and $\cell \subseteq \bigcup_{\simp \in \Delta_j} \simp$. Taking the union over cells, $\Delta_j$ is a triangulation of $\pset_j$.

Each $\simp \in \Delta_j \subseteq \Delta$ has empty open circumball with respect to $\pset$, hence with respect to $\pset_j \subseteq \pset$. So $\Delta_j$ is a Delaunay triangulation of $\pset_j$.
\end{proof}

\subsection{Proofs of Technical Claims}
\label{app:proofs}
\appsubsubsection{Theorem 1}
\label{app:thm1_proof}
\paragraph{Statement.}
If $f$ is enclosed by $\bset$, then for any $\al,\au \in \reals^n$, any $\x \in \conv(\pset)$, and any Delaunay triangulation $\Delta$ of \pset, $\Lc(\x) \leq \Ltri(\x) \leq f(\x) \leq \Utri(\x) \leq \Uc(\x)$.

\begin{proof}
    Begin by considering the vertices $\p \in \pset$. For any $\p \in \pset$, $\Lc(\p) = \al \cdot \p + \bl$. Since \bl is defined as the minimum, $\Lc(\p) \leq \al \cdot \p + (L(\p) - \al \cdot \p)$. Simplifying yields $\Lc(\p) \leq L(\p)$. Note that this holds for any $\al$.

    Now consider $\x \in \conv(\pset)$. Let $\simp = \simp_\Delta(\x)$ be a simplex in $\Delta$ containing \x. Let $V = (v_0,v_1,\ldots,v_n)$ denote the vertices of \simp, and $\theta = \theta_\Delta(\x)$ denote barycentric coordinates such that $\x = \theta \cdot V$. Then by definition, $L_\simp(\x) = \theta \cdot L(V)$, where $L_\simp$ is the restriction of $\Ltri$ to \simp. Since \Lc is affine, $\Lc(\x) = \theta \cdot \Lc(V)$. But from the previous paragraph, we know that $\Lc(v_i) \leq L(v_i)$ for all $v_i \in V$ (since each $v_i \in \pset$). Therefore $\Lc(\x) \leq L_\simp(\x) = \Ltri(\x)$. The soundness of the upper bound follows from symmetry, and the middle inequalities $\Ltri(\x) \leq f(\x) \leq \Utri(\x)$ are precisely the hypothesis that $f$ is enclosed by \bset.
\end{proof}

\appsubsubsection{Theorem 2}
\paragraph{Statement.}
Let $\bset$ be a bounding set, let $H$ be a grid hyperplane of $\pset$, and let $\bset^1, \bset^2$ be the refinements of $\bset$ along $H$. Let $\Delta$ be a Delaunay triangulation of $\pset$, and let
\[
    \Delta_j \coloneqq \{\simp \in \Delta \mid \simp \subseteq \conv(\pset_j)\} \quad \text{for } j \in \{1, 2\}.
\]
Each $\Delta_j$ is a Delaunay triangulation of $\pset_j$ by the Delaunay splitting lemma (see Appendix~\ref{app:delaunay_split}). Then:
\begin{enumerate}
    \item \textbf{(Unrelaxed.)} $\enc(\bset^1, \Delta_1) \cup \enc(\bset^2, \Delta_2) = \enc(\bset, \Delta)$.

    \item \textbf{(Relaxed.)} Let $\al, \au \in \reals^n$, let $\bset_\square$ and $\bset_\square^j$ be the relaxations of $\bset$ and $\bset^j$ with the same parameters $(\al, \au)$, and let $\bl^j, \bu^j$ denote the offsets of $\bset_\square^j$. For any Delaunay triangulations $\Delta_\square$ of $\pset_\square$ and $\Delta_\square^j$ of $\pset_\square^j$,
    \[
        \enc(\bset_\square^1, \Delta_\square^1) \cup \enc(\bset_\square^2, \Delta_\square^2) \subseteq \enc(\bset_\square, \Delta_\square),
    \]
    with strict containment whenever $(\bl^j, \bu^j) \neq (\bl, \bu)$ for some $j \in \{1, 2\}$.
\end{enumerate}

\begin{proof}
Let $\Delta_n \subseteq \Delta$ denote the set of $n$-simplices of $\Delta$.
By Delaunay Cell Containment (Appendix~\ref{app:delaunay_cells}), every $n$-simplex of $\Delta$ is contained in a grid cell of $\pset$. Each grid cell lies on one side of any grid hyperplane, hence in either $\conv(\pset_1)$ or $\conv(\pset_2)$, so every $n$-simplex of $\Delta$ does too. An $n$-simplex contained in both halves would lie in $\conv(\pset_1) \cap \conv(\pset_2) \subseteq H$, contradicting its dimension. Therefore $\Delta_n = \Delta_{1,n} \sqcup \Delta_{2,n}$, where $\Delta_{j,n} := \Delta_n \cap \Delta_j$.

\textbf{(Unrelaxed.)} For $\simp \in \Delta_n$, let $\bset|_\simp$ denote the restriction of $\bset$ to $\svert(\simp)$, and let $\poly(\bset|_\simp)$ denote the polyhedron formed by $\bset|_\simp$. Then $\enc(\bset, \Delta) = \bigcup_{\simp \in \Delta_n} \poly(\bset|_\simp)$. For $\simp \in \Delta_{j,n}$, the vertices of $\simp$ also lie in $\pset_j$, so $\bset|_\simp = \bset^j|_\simp$. Therefore
\begin{multline*}
    \enc(\bset, \Delta)
    = \bigcup_{\simp \in \Delta_n} \poly(\bset|_\simp) 
    = \bigcup_{j \in \{1,2\}} \bigcup_{\simp \in \Delta_{j,n}} \poly(\bset^j|_\simp) \\
    = \enc(\bset^1, \Delta_1) \cup \enc(\bset^2, \Delta_2).
\end{multline*}

\textbf{(Relaxed.)} Since $\bset_\square$ replaces $L, U$ with the affine functions $\Lc, \Uc$, the lower and upper surfaces of $\enc(\bset_\square, \Delta_\square)$ are $\Lc, \Uc$ on $\conv(\pset_\square) = \conv(\pset)$ independently of $\Delta_\square$. 
Hence
\begin{multline*}
    \enc(\bset_\square, \Delta_\square) = \{(\x, y) \in \reals^n \times \reals \mid \x \in \conv(\pset),\\ \Lc(\x) \leq y \leq \Uc(\x)\},
\end{multline*}
and analogously
\begin{multline*}
    \enc(\bset_\square^j, \Delta_\square^j) = \{(\x, y) \mid \x \in \conv(\pset_j),\\ \Lc^j(\x) \leq y \leq \Uc^j(\x)\}.
\end{multline*}
Since $\pset_j \subseteq \pset$, $\bl^j = \min_{p \in \pset_j}(L(p) - \al \cdot p) \geq \min_{p \in \pset}(L(p) - \al \cdot p) = \bl$. A symmetric argument yields $\bu^j \leq \bu$. Combined with $\conv(\pset_j) \subseteq \conv(\pset)$, this gives $\enc(\bset_\square^j, \Delta_\square^j) \subseteq \enc(\bset_\square, \Delta_\square)$ for $j \in \{1, 2\}$, hence $\enc(\bset_\square^1, \Delta_\square^1) \cup \enc(\bset_\square^2, \Delta_\square^2) \subseteq \enc(\bset_\square, \Delta_\square)$.

For strict containment, suppose $\bl^{j} > \bl$ for some $j \in \{1, 2\}$ (the case $\bu^{j} < \bu$ is symmetric). Pick $\x^*$ in the interior of $\conv(\pset_{j})$ and let $y^* = \Lc(\x^*)$. Then $(\x^*, y^*) \in \enc(\bset_\square, \Delta_\square)$, with the lower-bound inequality holding by construction and the upper-bound inequality since $\Lc \leq L \leq U \leq \Uc$ on $\pset$ gives $\Lc \leq \Uc$ on $\conv(\pset)$ by affinity. But $y^* < \Lc^{j}(\x^*)$ rules out $(\x^*, y^*) \in \enc(\bset_\square^{j}, \Delta_\square^{j})$, and $\x^*$ being interior to $\conv(\pset_{j})$ rules out $\x^* \in \conv(\pset_{j'})$ for $j' \neq j$. 
Therefore $(\x^*, y^*) \in \enc(\bset_\square, \Delta_\square) \setminus \big(\enc(\bset_\square^1, \Delta_\square^1) \cup \enc(\bset_\square^2, \Delta_\square^2)\big)$, so the union is a proper subset of $\enc(\bset_\square, \Delta_\square)$.
\end{proof}

\appsubsubsection{Proposition 1}

\paragraph{Statement.}
Let $\nfg$ have a BART relaxation at each RAIL node and a LiRPA
relaxation at each nonlinear node in every $\pi^{(t)}$. Then any
RAIL evaluation of $\nfg$ over $\init \times \noise^\horizon$
produces linear bounds on $\state_t$ that, when concretized,
sound-overapproximate $\trajj^\dynsys(\init)_t$ for every
$t \in [0..\horizon]$.

\begin{proof}
By \Cref{def:nfg}, the time $t$ subgraph of $\nfg$ realizes
$\next^\dynsys$ as a function of $(x_t, \epsilon_t)$, and chaining
the subgraphs through state nodes $x_1, \ldots, x_\horizon$ realizes
$\trajj^\dynsys$ on $\init \times \noise^\horizon$.

BART relaxations at enclosure nodes are sound by
\Cref{thm:relaxation}, LiRPA relaxations at nonlinear nodes in each $\pi^{(t)}$ are sound by~\Cref{prop:lirpa}, and
the affine layer is exact. Composition of per-node linear bounds
through the graph yields linear bounds on $x_t$ (the LiRPA framework, applied to the directed acyclic structure of
$\nfg$~\cite{xu2020automatic}).
Concretization of these bounds over $\init \times \noise^\horizon$
yields a set $\hat{\state}_t \supseteq \state_t =
\trajj^\dynsys(\init)_t$ for every $t \in [0..\horizon]$.
\end{proof}

\appsubsubsection{Proposition 2}

\paragraph{Statement.}
Let $\hat{\State}_0, \ldots, \hat{\State}_\horizon$ satisfy
$\hat{\State}_t \supseteq \State_t$ for all $t \in [0..\horizon]$.
If $\hat{\State}_t \cap \avoid(t) = \emptyset$ for all $t$ and
$\hat{\State}_{t^*} \subseteq \reach$ for some
$t^* \in [0..\horizon]$, then $\dynsys$ is safe.

\begin{proof}
For each $t \in [0..\horizon]$, $\State_t \subseteq \hat{\State}_t$,
so $\State_t \cap \avoid(t) \subseteq \hat{\State}_t \cap \avoid(t)
= \emptyset$, satisfying~\eqref{eq:avoid_prop}. Similarly,
$\State_{t^*} \subseteq \hat{\State}_{t^*} \subseteq \reach$,
satisfying~\eqref{eq:reach_prop}.
\end{proof}

\newpage
\section{Experimental Setup}
\label{app:experimental-setup}
\subsection{Benchmarks}
\label{app:benchmarks}
We evaluate all tools on the nine ARCH-COMP 2025 AINNCS benchmarks~\citep{manzanas_lopez2025arch}. The benchmarks were originally specified in continuous time. Following the guidance of the organizers, we obtain a discrete-time neural feedback system by forward Euler integration with step $\timestep$. For each benchmark we give a brief description, the tuple $\langle m, n, \init, \trans, \noise, \ctrl, B, \timestep, \horizon, \reach, \avoid \rangle$, and the reach-avoid specification. Where the controller acts on preprocessed features (e.g.\ relative distances) or outputs (e.g.\ $\pi(\state) - 10$), we fold these transforms into $\ctrl$, so that $\ctrl : \reals^n \to \reals^m$ maps the full state to the applied control and $B$ records only which coordinates the control drives.
\paragraph{Adaptive Cruise Control (ACC).}
An ego vehicle tracks a set speed while maintaining a safe distance from a lead vehicle that brakes abruptly. The state $\state = (x_l, v_l, a_l, x_e, v_e, a_e)$ collects the position, velocity, and acceleration of the lead ($l$) and ego ($e$) vehicles. We have $n = 6$, $m = 1$, and
\[
  \trans(\state) = \bigl(v_l,\; a_l,\; -2a_l - 4 - u v_l^2,\; v_e,\; a_e,\; -2a_e - u v_e^2\bigr),
\]
with drag coefficient $u = 10^{-4}$ and the constant lead deceleration $a_{c,l} = -2$ folded into the third component. The controller output enters the ego acceleration through $B = (0,0,0,0,0,2)^\top$. The controller $\ctrl$ is a five-layer network of $20$ ReLU units per layer whose fixed affine input layer forms the features $(v_\text{set}, T_\text{gap}, v_e, x_l - x_e, v_l - v_e)$ from the state, with constants $v_\text{set} = 30$ and $T_\text{gap} = 1.4$. The remaining components are
\[
  \init = [90,110]\times[32,32.2]\times\{0\}\times[10,11]\times[30,30.2]\times\{0\},
\]
$\noise = \{0\}^6$, $\timestep = 0.1$, and $\horizon = 50$ (a $5$\,s window). $\reach = \reals^6$ and $\avoid(t) = \{\state : x_l - x_e < 10 + 1.4\, v_e\}$.

\paragraph{TORA.}
A cart on a frictionless surface is stabilized by a rotational actuator. The state $\state = (x_1, x_2, x_3, x_4)$ collects the cart position, velocity, actuator angle, and angular velocity. We have $n = 4$, $m = 1$, and
\[
  \trans(\state) = \bigl(x_2,\; -x_1 + 0.1\sin x_3,\; x_4,\; 0\bigr).
\]
The control enters the actuator through $B = (0,0,0,1)^\top$, and $\noise = \{0\}^4$. The benchmark provides three controllers, giving instances TORA~1--3 in \Cref{tab:results}. TORA~1 (\emph{remain}) uses a network with three hidden layers of $100$ ReLU units and a linear output, with $\ctrl = \pi(\state) - 10$, $\timestep = 1$, $\horizon = 20$, $\init = [0.6,0.7]\times[-0.7,-0.6]\times[-0.4,-0.3]\times[0.5,0.6]$, $\reach = \reals^4$, and $\avoid(t) = \{\state : \state \notin [-2,2]^4\}$.
TORA~2 (\emph{reach-sigmoid}) and TORA~3 (\emph{reach-tanh}) use networks with three hidden layers of $20$ units, the former with sigmoid activations and $\ctrl = 22\pi(\state) - 11$, the latter with ReLU activations and a tanh output and $\ctrl = 11\pi(\state)$. The remaining parameters are $\timestep = 0.5$, $\horizon = 10$, $\init = [-0.77,-0.75]\times[-0.45,-0.43]\times[0.51,0.54]\times[-0.3,-0.28]$, $\reach = \{\state : x_1 \in [-0.1,0.2],\, x_2 \in [-0.9,-0.6]\}$, and $\avoid \equiv \emptyset$.

\paragraph{Unicycle.}
A unicycle model of a car drives on a plane. The state $\state = (x_1, x_2, x_3, x_4)$ collects the planar position, heading, and speed. We have $n = 4$, $m = 2$, and
\[
  \trans(\state) = \bigl(x_4 \cos x_3,\; x_4 \sin x_3,\; 0,\; 0\bigr).
\]
The controls drive the heading and speed through $B = \begin{psmallmatrix} 0 & 0 \\ 0 & 0 \\ 0 & 1 \\ 1 & 0 \end{psmallmatrix}$, and $\noise = \{0\}^3 \times 10^{-4}[-1,1]$ perturbs the speed. The controller is a single hidden layer of $500$ ReLU units with $\ctrl_i = \pi(\state)_i - 20$, $\timestep = 0.2$, $\horizon = 50$, $\init = [9.5,9.55]\times[-4.5,-4.45]\times[2.1,2.11]\times[1.5,1.51]$, $\reach = [-0.6,0.6]\times[-0.2,0.2]\times[-0.06,0.06]\times[-0.3,0.3]$, and $\avoid \equiv \emptyset$.

\paragraph{Single Pendulum.}
An inverted pendulum is stabilized by a torque at its pivot. The state $\state = (x_1, x_2) = (\theta, \dot\theta)$ collects the angle and angular velocity. We have $n = 2$, $m = 1$, and
\[
  \trans(\state) = \bigl(x_2,\; 2 \sin x_1\bigr).
\]
For parameters $m_p = L = 0.5$, $c = 0$, $g = 1$, the gravity term is $g/L = 2$ and the torque enters the angular acceleration through $B = (0, 8)^\top$ with gain $1/(m_p L^2) = 8$. The controller is a ReLU network with two hidden layers of $25$ units.
The remaining components are $\noise = \{0\}^2$, $\timestep = 0.05$, $\horizon = 20$, $\init = [1.0,1.175]\times[0.0,0.2]$, $\reach = \reals^2$, and $\avoid(t) = \{\state : \theta \notin [0,1]\}$ for $t \in [10,20]$ and $\emptyset$ otherwise.

\paragraph{Airplane.}
A simplified six-degree-of-freedom airplane model has state $\state = (s_x, s_y, s_z, v_x, v_y, v_z, \phi, \theta, \psi, r, p, q)$ of position, velocity, Euler angles, and body rates. We have $n = 12$ and $m = 6$ (three forces and three moments); $\trans$ and $B$ are the equations of motion of~\citep[eq.~10--12]{manzanas_lopez2025arch}, with parameters $m_a = I_x = I_y = I_z = 1$, $I_{xz} = 0$, $g = 1$. The controller is a ReLU network with three hidden layers of $100$, $100$, and $20$ units. With $\timestep = 0.1$ and $\horizon = 20$, the initial set fixes $s_x = s_y = s_z = r = p = q = 0$ and $[v_x, v_y, v_z, \phi, \theta, \psi] \in [0,1]^6$, and the specification requires $s_y \in [-1,1]$ and $\phi,\theta,\psi \in [-1,1]$ throughout.

\paragraph{Attitude Control.}
The attitude of a rigid body is controlled through three torques, with state $\state = (\omega, \psi) \in \reals^3 \times \reals^3$ of angular velocity and Rodrigues parameters. Here $n = 6$, $m = 3$, $\noise = \{0\}^6$, and the angular-velocity dynamics are $\dot\omega = (0.25(u_0 + \omega_2\omega_3),\, 0.5(u_1 - 3\omega_1\omega_3),\, u_2 + 2\omega_1\omega_2)$ with the controls entering through $B$; the $\dot\psi$ block is the kinematics of~\citep[Sec.~V]{manzanas_lopez2025arch}. The controller has three hidden layers of $64$ units with sigmoid activations and a linear output, $\timestep = 0.1$, and $\horizon = 30$. From the initial set of~\citep{manzanas_lopez2025arch} the specification requires avoiding the unsafe box given there: $\reach = \reals^6$ and $\avoid$ is that box for all $t$.

\paragraph{Quadrotor.}
A quadrotor with twelve states (position, velocity, attitude, body rates) is controlled by three torques. We have $n = 12$, $m = 3$, $\noise = \{0\}^{12}$, with $\trans$ and $B$ the dynamics of~\citep[eq.~12--16]{manzanas_lopez2025arch}. The controller has three hidden layers of $64$ units with sigmoid activations and a linear output, $\timestep = 0.1$, and $\horizon = 50$. The initial set has $x_1, \ldots, x_6 \in [-0.4,0.4]$ and $x_7, \ldots, x_{12} = 0$, and the reach specification stabilizes the altitude, $\reach = \{\state : x_3 \in [0.94,1.06]\}$, $\avoid \equiv \emptyset$.

\paragraph{2D Spacecraft Docking.}
A deputy spacecraft docks with a chief under linear Clohessy-Wiltshire relative dynamics, $\state = (s_x, s_y, \dot s_x, \dot s_y)$. Here $n = 4$, $m = 2$, $\noise = \{0\}^4$,
\[
  \trans(\state) = \bigl(\dot s_x,\; \dot s_y,\; 3\eta^2 s_x + 2\eta \dot s_y,\; -2\eta \dot s_x\bigr),
\]
with $\eta = 0.001027$, and the controls enter the accelerations through $B$ scaled by $1/m_s$, $m_s = 12$. The controller has two hidden layers of $256$ units with tanh activations and a linear output, $\timestep = 1$, and $\horizon = 40$. From the initial set $[70,106]^2 \times [-0.28,0.28]^2$ the specification is the velocity-limit constraint $\lVert(\dot s_x, \dot s_y)\rVert \le 0.2 + 2\eta \lVert(s_x, s_y)\rVert$, encoded as $\avoid$ for all $t$.

\paragraph{Navigation.}
A robot navigates to a goal while avoiding an obstacle, $\state = (x, y, \theta, \nu)$. We have $n = 4$, $m = 2$, $\noise = \{0\}^4$,
\[
  \trans(\state) = \bigl(\nu \cos\theta,\; \nu \sin\theta,\; 0,\; 0\bigr), \quad
  B = \begin{psmallmatrix} 0 & 0 \\ 0 & 0 \\ 1 & 0 \\ 0 & 1 \end{psmallmatrix}.
\]
The two controllers (each a ReLU network with two hidden layers of $64$ units and a tanh output, a standard one and an adversarially trained robust one) give instances Navigation~1 and~2. With $\timestep = 0.2$ and $\horizon = 30$, the initial set is $[2.9,3.1]^2 \times \{0\} \times \{0\}$, and the reach-avoid specification has $\reach = \{\state : x, y \in [-0.5,0.5]\}$ and $\avoid(t) = \{\state : x \in [1,2],\, y \in [1,2]\}$.

\subsection{Baselines}
\label{app:baselines}
Every tool analyzes the same discrete-time system. For a benchmark tuple $\langle m, n, \init, \trans, \noise, \ctrl, B, \timestep, \horizon, \reach, \avoid \rangle$, the shared plant is the explicit-Euler map
\[
  \state_{k+1} = \state_k + \timestep\bigl(\trans(\state_k) + B\,\ctrl(\state_k) + \err_k\bigr), \quad \err_k \in \noise,
\]
and each baseline is specialized to that map using its own set representation. Three of the five participants (CORA, immrax, CROWN-Reach) are continuous-time tools in their ARCH-COMP form, and we modify them to run on the discrete-time map. The continuous-time
machinery that exists solely to enclose a continuous flow, and therefore has no discrete counterpart, is dropped, and every other knob is left at the tool's own competition value. CORA needs no such removal, since \texttt{nonlinearSysDT} is a native discrete-time class; immrax needs none either, since \texttt{evolution="discrete"} is a first-class mode; only CROWN-Reach requires a source modification, described below. 

\paragraph{OvertPoly}\hspace{-1em}~\citep{akinwande2026polyhedralenclosuresefficientcombinatorial} (v0.1.1, commit \texttt{7743bc2}) is the combinatorial solver whose grid-based enclosure of the dynamics \clipper also adopts. It encodes the closed loop over the horizon as a single mixed-integer program: the piecewise-linear enclosure of each nonlinear dynamics term contributes binary piece-selection variables, and the \ReLU controller contributes binary activation variables. We run it on Julia 1.10.9 with Gurobi 13.0.2 through JuMP, at the solver's default settings with logging suppressed, setting no optimality gap, thread cap, or time limit. The tool is already discrete-time, so no modification is required.

\paragraph{OVERTVerify}\hspace{-1em}~\citep{sidrane2022overt} (commit \texttt{96c4d95}) is a distinct combinatorial tool that shares OvertPoly's overapproximation of univariate dynamics but differs in the MILP encoding and in how the reachable set is queried. It is likewise natively discrete-time and runs on the same Julia and Gurobi versions with default solver settings.

\paragraph{Controller restrictions in the combinatorial solvers.}
Both MILP encodings require a piecewise-linear controller, which they realize as a big-M encoding of \ReLU activations. This accounts for seven of the eight \texttt{U} cells in \Cref{tab:results}, whose controllers use sigmoid or tanh activations. The eighth, Airplane, has a \ReLU controller, but at $n = 12$ and $m = 6$ the resulting encoding is too large for either tool to handle. \clipper adopts the same grid-based enclosure of the dynamics but relaxes the controller with LiRPA rather than encoding it, and so places no restriction on its activations.

\paragraph{CORA}\hspace{-1em}~\citep{kochdumper2023open} (v2026.1.0, commit \texttt{904a444}) is pinned as a submodule and kept pristine. We build the Euler plant as a \texttt{nonlinearSysDT}, compose the controller with \texttt{neurNetContrSys}, and run reach, specification check, and CORA's own recursive initial-set splitting. We do not call its \texttt{verify} entry point, whose \texttt{simulateRandom} stage can return a falsifying trajectory before any splitting occurs and would report falsification time in place of set computation time. Set-representation options are taken per benchmark from CORA's own ARCH-COMP scripts, though options which have no discrete counterpart are not set.
 
\paragraph{immrax}\hspace{-1em}~\citep{harapanahalli2024immrax} (commit \texttt{af79191}) performs mixed-monotone interval reachability in JAX, bounding the controller with CROWN through \texttt{jax\_verify}. It required no modification: discrete evolution is a first-class mode and the mixed-monotone embedding has a discrete branch, so the shared Euler map is expressed directly as an immrax system and rolled out over the horizon. The only code we wrote is an adapter, including a converter from ONNX to immrax's network format, validated against our own controller loader to within $10^{-4}$. Analysis options match immrax's published neural-network configuration, and where a benchmark warrants initial-set partitioning we use immrax's own pipeline: its documented $15 \times 15$ position grid on the Navigation instances, and a two-cell heading split on Unicycle mirroring the partition \clipper uses there.

\paragraph{CROWN-Reach}\hspace{-1em}~\citep{manzanas_lopez2024arch,crown_reach} (commit \texttt{7b90f30}) bounds the controller with LiRPA and propagates the plant as a Taylor model through Flow$^*$. It is the only participant with no discrete-time mode, and the only one whose source we changed, through three tracked patches against a pinned submodule. The substantive patch replaces Flow$^*$'s continuous flowpipe integration with the shared Euler map applied directly to the Taylor model, which drops the a-priori Picard enclosure and the flow remainder; both exist only to enclose the interior of an integration step, and discrete-map Taylor-model reachability is a standard construction. Taylor-model conservatism is untouched, as remainder growth and wrapping across steps remain and the tool keeps its own order and remainder settings. The remaining patches disable the sampling-based falsifier and add configurations for the instances CROWN-Reach does not ship; every reachability option is the shipped ARCH-COMP value.

\paragraph{Configuration audit.}
Because untuned baselines are the obvious threat to the timing and tightness comparisons, we audited every option our harness sets against the value used by each tool's own competition scripts. CROWN-Reach was already running its shipped configurations verbatim, and immrax matched its published configuration everywhere except Unicycle, where we then gave it the same heading split \clipper uses. CORA was running one global default where its scripts tune per benchmark, which left its verdicts unchanged but inflated its reported volumes; every CORA row in \Cref{tab:results} was regenerated under the per-benchmark configuration. The one remaining gap is ACC, which we run at a lower zonotope order than CORA's script specifies because that value crashes CORA partway through the horizon.

\subsection{Implementation and Compute}
\label{app:implementation}
All AINNCS experiments were run on a single workstation with an AMD Ryzen 9 7950X CPU (16 cores, 32 threads), 64\,GB of RAM, and an NVIDIA GeForce RTX 3060 GPU with 12\,GB of VRAM (driver 580.159.03), running Ubuntu 24.04.4 LTS. \clipper is implemented in Python 3.12.13 against PyTorch 2.8.0 with CUDA 12.6, \texttt{auto\_LiRPA} 0.7.0, and NumPy 2.3.5. The environment is pinned by a \texttt{pixi.lock} file that will be released with the source code. The image-based experiments of \Cref{app:aebs} were run separately on an NVIDIA H100 80\,GB GPU and took $120$ GPU-hours in total. Enclosure-slope optimization runs for 50 gradient steps; the relaxation ablation of \Cref{tab:relaxation} holds this budget fixed across all three relaxations.

Across the AINNCS benchmarks, \muni selects enclosure refinements exclusively and no neuron split is taken. Their controllers are small relative to the dynamics relaxation, and so a refinement always scores higher on the shared scale. This is because the bulk of the bound-sensitive relaxation error comes from the plant. For example, on the unicycle benchmark with 500 neurons, we only encounter 7 ``unstable'' neurons during the initial concrete reachability pass, and neuron splitting is not beneficial there. 
\clearpage
\section{Image-Based Neural Feedback Systems}
\label{app:aebs}
The AINNCS benchmark suite is limited to systems with small, state-based controllers. This favors tools with control-native abstractions, though \clipper remains competitive regardless. To further highlight the benefits of our approach, we evaluate \clipper on the advanced emergency braking system of~\citet{cai2025scalable}, which uses a image based controller. 
\subsection{The Advanced Emergency Braking System}
The state is the distance to a stopped obstacle and the ego velocity, $\state = (d, v) \in [0,60] \times [0,30]$, and the plant is linear, $d^+ = d - v\timestep$ and $v^+ = v - a\timestep$ with $\timestep = 0.05$ and a braking deceleration $a$ affine in the controller output. The perception system is replaced by a GAN conditioned on $d$, carrying four latent variables bounded by $10^{-2}$ in magnitude and redrawn at every control step. The network maps $(d, z, v)$ to a braking force through a generator and a controller composed into one convolutional network. The specification requires the vehicle to stop before $d$ reaches zero.

\subsection{Inconclusive Instances}
\citet{cai2025scalable} partition the state space into a $100 \times 100$ grid and prune it with $5000$ simulations per cell, then prove surviving cells safe by unrolling $m$ control steps into a single composed network, bounding it with $\alpha,\beta$-CROWN, and re-abstracting the result back to grid cells. Coverage grows with $m$, but so does the composed network, and $m = 3$ is the ceiling they report. On the $20$\,Hz convolutional controller this leaves $3868$ cells ($38.7\%$ of the space), that are neither falsified by simulation, nor proven safe at any $m$ they can reach (\Cref{tab:aebs}). 

\subsection{Reducing Inconclusive Cells}
\clipper targets these cells directly and does not re-run the verificed/falsified cells. We discuss the three routes close a cell, and \Cref{tab:aebs} reports how many cells each one settles.

\paragraph{Kinematic lemma.}
We can exploit two facts about the benchmark to resolve inconclusive cells without network analysis. The first is that the controller's output layer clamps outputs to $[0,1]$. The second is that the velocity dynamics are affine, so every step sheds at least the speed the drag term alone removes. A cell is therefore safe whenever its clearance exceeds the distance it could travel under the floor set by the most adversarial controller. By construction, the worst case control is zero, and we can evaluate the floor in microseconds. 

\paragraph{Concrete reachability.}
One step of the composed network and plant is bounded with \clipper from a state box and the latent box, giving a box that contains the true one-step image, and iterating this yields a certified tube. A cell closes when the velocity box falls to zero with clearance maintained throughout, or when clearance is maintained through the step bound the kinematic lemma floor supplies. 
\paragraph{Symbolic reachability.}
When that is not sufficient, \clipper bounds the unrolled composition from the failing step back to the cell in a single backward pass, carrying one latent symbol per step and re-boxing nowhere in between. Every intermediate pre-activation interval is intersected with the one the concrete pass computed for the same node, so the joint bound is no looser than the concrete bound it replaces, and we report the better of the two at each step. 

We branch on input and interior state nodes. We expect the cells that remain to need branching on the controller's activations alongside these routes, an action \clipper supports but which \muni does not select on any system reported here.

\begin{table}[t]
\centering
\footnotesize
\caption{The $20$\,Hz convolutional configuration of~\citet{cai2025scalable} and the part of its inconclusive band \clipper resolves. Counts are cells of their $100 \times 100$ grid, and their reported cost for this configuration is $374$ GPU-hours normalized to a single machine. The three routes overlap, so the rows attributing cells to them need not sum to the total resolved.}
\label{tab:aebs}
\setlength{\tabcolsep}{4pt}
\begin{tabular}{l r}
\hline
Cells in the state space & $10{,}000$ \\
Unsafe by simulation & $3463$ \\
Proven safe, $1$-step & $0$ \\
Proven safe, $2$-step & $384$ \\
Proven safe, $3$-step & $2669$ \\
Inconclusive after $3$-step & $3868$ \\
\hline
Settled by the kinematic lemma & 1105 \\
Settled by concrete reachability & 44 \\
Settled by symbolic reachability & 9 \\
Resolved by \clipper ($\bigstar$) & 1116 \\
Left inconclusive & 2752 \\
\hline
\end{tabular}
\end{table}

\end{document}